\documentclass[hidelinks,reqno,11pt]{article}
\usepackage{fullpage,amsmath,amssymb,enumerate,enumitem,mathtools,float,accents,chngcntr,hyperref,url}
\usepackage{amsthm}

\usepackage{algorithm, algorithmic}

\numberwithin{equation}{section}

\newtheorem{theorem}{Theorem}[section]
\newtheorem*{theorem*}{Theorem}
\newtheorem{lemma}[theorem]{Lemma}
\newtheorem*{lemma*}{Lemma}

\theoremstyle{definition}
\newtheorem{definition}[theorem]{Definition}
\newtheorem{remark}[theorem]{Remark}

\DeclarePairedDelimiter\abs{\lvert}{\rvert}
\DeclarePairedDelimiter\norm{\lVert}{\rVert}
\DeclarePairedDelimiter\paren{(}{)}
\DeclarePairedDelimiter\braces{\lbrace}{\rbrace}
\DeclarePairedDelimiter\inprod{\langle}{\rangle}
\DeclarePairedDelimiter\sqbracket{[}{]}

\newcommand{\Prob}[2][]{\mathbb{P}_{#1}\paren*{#2}}

\newcommand{\E}{\mathbb{E}}
\newcommand{\R}{\mathbb{R}}

\newcommand{\tensor}{\otimes}

\newlength{\dhatheight}

\newcommand{\st}{\ | \ }

\newcommand{\tr}[1]{\text{tr}\paren{#1}}

\title{Polynomial Time and Sample Complexity for Non-Gaussian Component Analysis: Spectral Methods}

\author{Yan Shuo Tan \footnote{Department of Mathematics, University of Michigan, \href{mailto:yanshuo@umich.edu}{yanshuo@umich.edu}}  \quad\quad\quad Roman Vershynin \footnote{Department of Mathematics, University of Michigan, \href{mailto:romanv@umich.edu}{romanv@umich.edu}}}

\begin{document}

\maketitle

\begin{abstract}
	The problem of Non-Gaussian Component Analysis (NGCA) is about finding a maximal low-dimensional subspace $E$ in $\R^n$ so that data points projected onto $E$ follow a non-gaussian distribution. Although this is an appropriate model for some real world data analysis problems, there has been little progress
	on this problem over the last decade.

	In this paper, we attempt to address this state of affairs in two ways. First, we give a new characterization of standard gaussian distributions in high-dimensions, which lead to effective tests for non-gaussianness. Second, we propose a simple algorithm, \emph{Reweighted PCA}, as a method for solving the NGCA problem.
	We prove that for a general unknown non-gaussian distribution, this algorithm recovers at least one 
	direction in $E$, with sample and time complexity depending polynomially on the dimension 
	of the ambient space. We conjecture that the algorithm actually recovers the entire $E$. 
\end{abstract}

\clearpage

\section{Introduction}

\subsection{Non-Gaussian Component Analysis} Dimension reduction is a necessary step for much of modern data analysis, the principle being that the structure or ``interestingness'' of a collection of data points is contained in a geometric structure which has much lower dimension than the ambient vector space. We consider the case where the geometric structure in question is a linear subspace. In other words, we are in the situation where the variation of the data points within this subspace contains some information which we would like to extract, while their variation in the complementary directions constitute mere noise.

In many cases, it is reasonable to think of the noise as being gaussian. Formally, we then have the following generative model. Let $E$ be an unknown $d$-dimensional subspace of $\R^n$, and let $E^\perp$ be the orthogonal complement of $E$. 
Let $X$ be a random vector in $\R^n$, which we can decompose into two independent components: a non-gaussian component $\tilde{X}$ that takes values in $E$, and a gaussian component $g$ that takes values in $E^\perp$. In other words, we let $X = \paren{\tilde{X},g} \in E \oplus E^\perp$.\footnote{It is not necessary to assume that the gaussian and non-gaussian subspaces are perpendicular. They automatically become perpendicular if we apply a whitening transformation.}

Our goal is to recover the subspace $E$ from a sample of independent realizations of $X$. This is precisely the framework of the problem of Non-Gaussian Component Analysis (NGCA). We make no assumption on the relative magnitudes of $\tilde{X}$ and $g$. When the noise component is much smaller, which is a reasonable assumption in some real world applications, $E$ can be recovered using the standard Principal Component Analysis (PCA). However, PCA manifestly fails when the signal to noise ratio is small, i.e. when $\tilde{X}$ has lower magnitude than $g$.

With mild distributional assumptions, applying a whitening transformation to the data points can be done efficiently with sample size linear in the dimension \cite{Vershynin2011b}. As such, we might as well assume that the distribution is already whitened (i.e. isotropic). In other words, for the rest of this paper, we work with the model:

\begin{definition}[Isotropic NGCA model]			\label{def: NGCA model}
	\begin{equation} \label{NGCA model}
	X = \paren{\tilde{X},g} \in E \oplus E^\perp, \quad  \E X = 0, \quad \E XX^T = I_n.
	\end{equation}
\end{definition}

The NGCA problem is closely related to the problem of Independent Component Analysis (ICA), but generalizes it in a crucial way. ICA assumes the existence of a latent variable $s$ with independent coordinates, whereas in our case, the distribution of $\tilde{X}$ is allowed to have any manner of dependencies amongst its entries.

\subsection{Previous work on NGCA} 

As far as we know, the NGCA problem was first studied by Blanchard et al. in their 2006 paper \cite{Blanchard2006b}. Following this paper, there has been a small but growing body of work on the problem \cite{Kawanabe2006b,Kawanabe2007b,Kawanabe2006,Diederichs2010b,Diederichs2013b,Sasaki2016b,Virta2016b}, most of which is based on adapting ideas from Projection Pursuit, and in particular, the iterative \emph{FastICA} algorithm of Hyv{\"a}rinen \cite{Hyvarinen1997}.

None of these works have provided an algorithm with provable polynomial time and sample complexity. Suggested algorithms either have scant theoretical justification, or have analyses that are unsatisfactory for two reasons. First, where there is a consistency guarantee for an estimator $\hat{E}$, the guarantee is either conditional on some uncheckable, idiosyncratic assumptions on how the particular algorithm interacts with the ``non-gaussianness'' of $\tilde{X}$ (for instance, see Assumption 1 in \cite{Diederichs2013b}), or states only that it converges to a certain subspace $F$ that is related to $E$, but may not coincide with $E$ itself. Second, rates of convergence either omit dependence on the dimension or are again given in terms of idiosyncratic assumptions.

Furthermore, the algorithms with better theoretical grounding, such as in \cite{Diederichs2013b,Sasaki2016b}, have also become incredibly complex, and one hopes that simpler algorithms should suffice.

\subsection{Reweighted PCA in other contexts} 
Our proposed algorithm, Reweighted PCA, will involve performing PCA on a weighted sample. The idea of doing this with weight functions that are non-linear in the sample points can be traced back at least as far as Brubaker and Vempala's work on \emph{Isotropic PCA} \cite{Brubaker2008b}. In that paper, the authors similarly use gaussian weights, but do so in order to handle clustering for gaussian mixture models that are highly non-spherical. In a later paper \cite{Goyal2014a}, Goyal, Vempala, and Xiao used Fourier weights to handle ICA. While our analysis is radically different, the idea for the algorithm was directly inspired by these two papers.

\section{Main results} \label{sec: main results}

The principle that underlies our approach to NGCA is a new characterization 
of multi-dimensional gaussian distributions. 
Throughout this section, $X$ denotes a random vector in $\R^n$
and $g$ is a standard gaussian random vector in $\R^n$.
By $X'$ we will always denote an independent copy of $X$.

\begin{theorem}[First gaussian test] \label{thm: first gaussian test}
  Suppose $X$ has the same radial distribution as $g$, i.e. $\|X\|_2$ and $\|g\|_2$
  are identically distributed. If $\inprod{X,X'}$ has the same distribution 
  as $\inprod{g,g'}$, then $X$ has the same distribution as $g$, i.e. 
  the standard gaussian distribution.
\end{theorem}

We will prove this theorem in Section \ref{sec: proof of first gaussian test} 
via a decomposition of moment tensors and a resulting energy minimization property 
of the gaussian measure. The theorem guarantees that any non-gaussianness of $X$ is always captured by 
either the norm $\norm{X}_2$ or the dot product pairings $\inprod{X,X'}$. It is clear that the norm condition on its own is not sufficient to guarantee that $X \stackrel{d}{=} g$. For instance, let $\theta$ have any non-uniform distribution on the sphere. Then $\norm{g}_2\theta$ has the same radial distribution as $g$, but is not itself gaussian.

This result by itself does not address the NGCA problem, in which 
we are looking for non-gaussian {\em directions} in the distribution of $X$.
To this end, we propose a matrix version of the first gaussian test. 
Pick a parameter $\alpha>0$ and consider the {\em test matrices} 
\begin{equation}		\label{eq: test matrices}
	\Phi_{X,\alpha} := \frac{1}{Z_\Phi} \, \E e^{-\alpha\norm{X}_2^2} X X^T 
	\quad \text{and} \quad
	\Psi_{X,\alpha} := \frac{1}{Z_\Psi} \, \E e^{-\alpha \inprod{X,X'}} X (X')^T,
\end{equation}
where the normalizing quantities $Z_\Phi = Z_{\Phi,X}(\alpha):= \E e^{-\alpha\norm{X}_2^2}$ 
and $Z_\Psi = Z_{\Psi,X}(\alpha) := \E e^{-\alpha \inprod{X,X'}}$ resemble partition functions in statistical mechanics.

For a standard gaussian random vector $g$, a straightforward computation (see Lemma \ref{lem: formula for Phi and Psi for gaussian}) shows that both test matrices are multiples of the identity, namely
\begin{equation}		\label{eq: test matrices for gaussian}
  \Phi_{g,\alpha} = (2\alpha+1)^{-1} I_n
  \quad \text{and} \quad
  \Psi_{g,\alpha} = \alpha(\alpha^2-1)^{-1} I_n.
\end{equation}
Our second test guarantees that the non-gaussianness of $X$ is captured by 
one of the test matrices, and moreover that their eigenvectors reveal the non-gaussian directions of $X$.

\begin{theorem}[Second gaussian test] \label{thm: second gaussian test}
  Consider a random vector $X$ which follows the isotropic NGCA model \eqref{NGCA model},
  in which the projection $\tilde{X}$ of $X$ onto $E$ is not gaussian. 
  Then, for any $\abs*{\alpha}$ small enough, 
  either $\Phi_{X,\alpha}$ has an eigenvalue not equal to $(2\alpha+1)^{-1}$ 
  or $\Psi_{X,\alpha}$ has an eigenvalue not equal to $\alpha(\alpha^2-1)^{-1}$. 
  Furthermore, all eigenvectors corresponding to such eigenvalues lie in $E$.\footnote{The matrix $\Phi_{X,\alpha}$ always exists, but when $\tilde{X}$ is not subgaussian (i.e. can be rescaled so that marginals have tails lighter than a standard gaussian), $\Psi_{X,\alpha}$ may not be well-defined even for small $\alpha$. In that case, $\norm{\tilde{X}}_2$ has a different distribution from $\norm{g}_2$, so that $\Phi_{X,\alpha}$ has non-gaussian eigenvalues. We can hence think of $\Phi_{X,\alpha}$ as the primary test matrix, and $\Psi_{X,\alpha}$ being an auxiliary that is only required in hard (effectively adversarial) cases.}
\end{theorem}

In Section \ref{sec: proof of second gaussian test}, we will show how to derive the second gaussian test from the first using a block diagonalization formula for each of the matrices $\Phi_{X,\alpha}$ and $\Psi_{X,\alpha}$. Again, it is easy to see that $\Phi_{X,\alpha}$ is not sufficient by itself to identify non-gaussian directions: Take $X = \norm{g}_2\theta$ as before, and this time that assume that $\theta$ is uniform on $\braces{\pm e_i}_{1=1}^N$. The symmetry implies that $\Phi_{X,\alpha}$ is a scalar matrix, and computing its trace shows that it is equal to $(2\alpha+1)^{-1} I_n$.

For simplicity, we stated both gaussian tests for population rather than for finite samples;
they involve taking expectations over the entire distribution of $X$ which is typically unknown in practice.
However, both tests are quite robust and work provably well on finite (polynomially large) samples. 
Robust versions of gaussian tests can be formulated in terms of {\em deviations from gaussian moments}
\begin{equation} \label{eq: D_X}
D_{X,r} := \sup_{v \in S^{n-1}} \abs*{\E{\inprod{X,v}^r} - \E{\inprod{g,v}^r}}.
\end{equation}
  
Note that by the classical moment problem, if $D_{X,r}=0$ for all positive integers $r$, 
then $X$ has the same distribution as $g$, i.e. the standard gaussian distribution. 
The magnitude of $D_{X,r}$ can thus serve as a convenient quantitative measure 
of non-gaussianness of $X$. 

\begin{theorem}[First gaussian test, robust] \label{thm: first gaussian test, robust}
  There is a universal constant $c>0$ such that for each positive integer $r$, 
  we have either
  $$
  \abs*{\E{\norm{X}_2^r}-\E{\norm{g}_2^r}} \geq c \eta_r^2 /\gamma_r
  \quad \text{or} \quad
  \abs*{\E\inprod{X,X'}^r - \E\inprod{g,g'}^r} \geq c \eta_r^2.
  $$
  Here $\gamma_r = \E|\inprod{g,v}|^r$ for an arbitrary vector $v \in S^{n-1}$ 
  and $\eta_r = \min\braces{D_{X,r}, \gamma_r}$.  
\end{theorem}

As with the non-robust version, we will prove this theorem in Section \ref{sec: proof of first gaussian test} using a decomposition of moment tensors.
There is a similar robust version of the second gaussian test, which we will skip here
but state and prove in Section \ref{sec: proof of second gaussian test}.

Robustness allows us to use finite sample averages instead
of expectations in the gaussian tests, which is critical for practical applications. 
Indeed, consider a sample $X_1,\ldots,X_N, X'_1,\ldots,X'_N$ 
of $2N$ i.i.d. realizations of a random variable $X$.
We can then define the sample versions of the test matrices in \eqref{eq: test matrices} 
in an obvious way:
\begin{equation}
	\hat{\Phi}_{X,\alpha} = \frac{1}{\hat{Z}_\Phi} \, \sum_{i=1}^N e^{-\alpha\norm{X_i}_2^2} X_i X_i^T 
	\quad \text{and} \quad
	\hat{\Psi}_{X,\alpha} = \frac{1}{\hat{Z}_\Psi} \, \sum_{i=1}^N e^{-\alpha \inprod{X_i,X'_i}} (X_i (X'_i)^T + X_i'X_i^T),
\end{equation}
with the normalizing quantities $\hat{Z}_\Phi := \sum_{i=1}^N e^{-\alpha\norm{X_i}_2^2}$
and $\hat{Z}_\Psi := 2\sum_{i=1}^N e^{-\alpha \inprod{X_i,X'_i}}$.

The second gaussian test leads to the following straightforward algorithm for solving NGCA problem
based on a finite sample:
Use the sample to compute the test matrices $\hat{\Phi}_{X,\alpha}$ and $\hat{\Psi}_{X,\alpha}$;
select the eigenspaces corresponding to the eigenvalues that significantly 
deviate from the gaussian eigenvalues.
Then all vectors in both eigenspaces will be close to the non-gaussian subspace $E$ which we are trying to find. Let us state this algorithm and its guarantees precisely. 

\begin{algorithm}[H]                      
\caption{{\sc Reweighted PCA($X$,$\alpha_1$,$\alpha_2$,$\beta_1$,$\beta_2$)}}     
\begin{algorithmic}[1]              
    \REQUIRE Data points $X_1,\ldots,X_N, X'_1,\ldots,X'_N$, 
    	scaling parameters $\alpha_1,\alpha_2 \in \R$, 
	tolerance parameters $\beta_1,\beta_2 > 0$.
    \ENSURE Two estimates $\hat{E}_\Phi$ and $\hat{E}_\Psi$ for $E$.
    \STATE Compute test matrices $\hat{\Phi}_{X,\alpha_1}$ and $\hat{\Psi}_{X,\alpha_2}$.
    \STATE Compute the eigenspace $\hat{E}_\Phi$ of $\hat{\Phi}_{X,\alpha_1}$ 
    	corresponding to the eigenvalues that are farther than $\beta_1$ from the value $(2\alpha_1+1)^{-1}$.
    \STATE Compute the eigenspace $\hat{E}_\Psi$ of $\hat{\Psi}_{X,\alpha_2}$ 
    	corresponding to the eigenvalues that are farther than $\beta_2$ from 
	the value $\alpha_2(\alpha_2^2-1)^{-1}$.
\end{algorithmic}
\end{algorithm}

\begin{theorem}[Guarantee for Reweighted PCA] \label{PCA guarantee}
  Let $X$ be a subgaussian\footnote{For a formal definition of subgaussian random vectors and an introduction to their properties, please see \cite{Vershynin2011b}.} random vector which follows the isotropic NGCA model \eqref{NGCA model}, 
  and with subgaussian norm bounded above by $K \geq 1$. 
  Let $r$ be the integer such that $D_{\tilde{X},r} =: D > 0$ and $D_{\tilde{X},r'} = 0$ for all $r' < r$. Then for any $\delta,\epsilon \in (0,1)$, with probability at least $1-\delta$, if we run {\sc Reweighted PCA} with a choice of parameters $\alpha_1,\alpha_2,\beta_1,\beta_2$ that is optimal up to constant multiples, at least one of $\hat{E}_\Phi$ and $\hat{E}_\Psi$ is non-trivial, and any unit vector in their union is $\epsilon$-close to one in $E$, so long as the sample size $N$ is greater than $\text{poly}_r(n,1/\epsilon,\log(1/\delta),1/D,K)$. Here, $\text{poly}_r$ is a polynomial whose total degree depends linearly on $r$. 
\end{theorem}

 The idea of the proof is to use Davis-Kahan eigenvector perturbation theory \cite{Davis1970a}. The robust version of the second gaussian test exerts the existence of a gap between gaussian and non-gaussian eigenvalues. By bounding the deviation of the test matrix estimators $\hat{\Phi}_{X,\alpha}$ and $\hat{\Psi}_{X,\alpha}$ from their expectation, we can thus show that their eigenstructures are similar. We will prove this theorem formally in Section \ref{sec: proof of guarantee}.

Note that {\sc Reweighted PCA} is a simple spectral algorithm, which obviously runs in polynomial time. The name of the algorithm stems from the first test matrix, which can be seen as a PCA matrix for the reweighted sample obtained when each point $X_i$ is given the weight $e^{-\alpha\norm{X_i}_2^2}$. As mentioned earlier in the section, $\Phi_{X,\alpha}$ reveals at least one non-gaussian direction in all but adversarial situations, and so can be considered the primary test matrix.

\subsection{Organization of paper and notation}

In Section \ref{sec: proof of first gaussian test}, we will prove the first gaussian test and its robust version. In Section \ref{sec: proof of second gaussian test}, we will prove the second gaussian test and state a robust version needed for proving our guarantee for Reweighted PCA. The guarantee is then proven in Section \ref{sec: proof of guarantee}. For the sake of space, many technical details are deferred to the appendix. Throughout the paper, $C$ and $c$ denote absolute constants whose value may change from line to line. We let $g_n$ denote the standard gaussian vector in $\R^n$. The subscript is omitted whenever the dimension is obvious.

\section{Proof of the first gaussian test} \label{sec: proof of first gaussian test}

The first gaussian test is based on the first author's work on eccentricity tensors in \cite{Tan2016}. For completeness, we will repeat the key arguments. The statements and proofs in this section are valid more generally for random variables $X$ with finite moments of all orders (not necessarily subgaussian).

Recall the following fact from linear algebra. For any positive integer $r$, we may identify the $r$-th tensor product $T^r(\R^n) = \R^n\tensor\cdots\tensor\R^n$ with $\R^{n^r}$ by picking as a basis the vectors $\braces{e_{i_1}\tensor e_{i_2}\tensor\cdots\tensor e_{i_r}}_{1\leq i_1,\ldots i_r \leq n}$. With this choice, the Euclidean inner product between any two pure tensors $u_1\tensor\cdots\tensor u_r$ and $v_1\tensor\cdots\tensor v_r$ can be written as
\[
\inprod{u_1\tensor\cdots\tensor u_r,v_1\tensor\cdots\tensor v_r} = \prod_{i=1}^r\inprod{u_i,v_i}.
\]
In particular, for power tensors $u^{\tensor r}$ and $v^{\tensor r}$, we have the formula $\inprod{u^{\tensor r},v^{\tensor r}} = \inprod{u,v}^r$.

Now let $X$ and $Y$ be two independent random vectors. The above formula allows us to rewrite the $r$-th moment of their inner product as an inner product between their $r$-th moment tensors. Namely, we have
\begin{align} \label{tensorid1}
\E{\inprod{X,Y}^r} = \E{\inprod{X^{\tensor r},Y^{\tensor r}}} = \inprod*{M^r_X,M^r_Y},
\end{align}
where we define $M^r_X := \E X^{\tensor r}$.
For independent copies $X$, $X'$ of the same random vector having distribution $\mu$, $M^r_X = M^r_{X'}$, so
\begin{align} \label{tensorid2}
\E{\inprod{X,X'}^r} = \norm{M^r_X}^2.
\end{align}

Next, for any random vector $X$, let $X_{rot}$ denote a random vector that is independent of $X$, has the same radial distribution as $X$, and whose distribution is rotationally invariant. We call $X_{rot}$ the \emph{rotational symmetrization} of $X$. Comparing the moment tensors of a random vector and those of its rotational symmetrization give rise to what we shall call eccentricity tensors. Specifically, for any positive integer $r$, we define the $r$-th \emph{eccentricity tensor} of $X$ to be $E^r_X = M^r_X - M^r_{X_{rot}}$.

Since $X \stackrel{d}{=} X_{rot}$ if and only if $X$ is rotationally invariant, we see that the eccentricity tensors of $X$ are quantitative measures of how far its distribution is from being rotationally invariant. This interpretation is further supported by the following observation.

\begin{lemma}[Orthogonality]
	The eccentricity tensors of a random vector $X$ are orthogonal to the moment tensors of its rotational symmetrization. In other words, for any positive integer $r$,
	\begin{align} \label{eq: orthogonality}
	\inprod{E^r_X, M^r_{X_{rot}}} = 0
	\quad \text{and} \quad
	\norm*{M^r_X}_2^2 = \norm*{M^r_{X_{rot}}}_2^2 + \norm*{E^r_X}_2^2.
	\end{align}
\end{lemma}

\begin{proof}
	Let $Q$ be a random orthogonal matrix chosen according to the Haar measure on $O(n)$. For any fixed vector $v \in \R^n$, $Qv$ is uniformly distributed on the sphere of radius $\norm{v}_2$, so if $Y$ is any random vector independent of $Q$, applying $Q$ to $Y$ preserves its radial distribution but makes $QY$ rotationally invariant.
	
	Now choose $Q$ to be independent of $X$ and $X_{rot}$. Our previous discussion implies that $Q^TX \stackrel{d}{=} QX_{rot} \stackrel{d}{=} X_{rot}$.
	We use this to compute
	\begin{align}
	\E{\inprod{X,X_{rot}}^r} = \E{\inprod{X,QX_{rot}}^r} = \E{\inprod{Q^T X,X_{rot}}^r} = \E{\inprod{X_{rot}',X_{rot}}^r},
	\end{align}
	where $X_{rot}'$ is an independent copy of $X_{rot}$. We may then apply identities \eqref{tensorid1} and \eqref{tensorid2} to rewrite the above equation as
	\begin{align}
	\inprod*{M^r_X,M^r_{X_{rot}}} = \inprod*{M^r_{X_{rot}},M^r_{X_{rot}}}.
	\end{align}
	Subtracting the right hand side from the left hand side gives \eqref{eq: orthogonality}.
\end{proof}

\begin{theorem}\label{thm: inner product moments minimized by rotational symmetrization}
	Let $X$ be a random vector in $\R^n$ with finite moments of all orders. Then
	\begin{enumerate}[nosep]
		\item[a)] (Minimization) If $X'$ is an independent copy of $X$, and $X_{rot}, X_{rot}'$ are independent copies of its rotational symmetrization, we have
		\begin{align} \label{minimization}
		\E{\inprod{X,X'}^r} \geq \E{\inprod{X_{rot},X_{rot}'}^r}
		\end{align}
		for any positive integer $r$.
		\item[b)] (Uniqueness) Furthermore, if equality holds in \eqref{minimization} for all $r$ and we further assume that $X$ has a subexponential distribution\footnote{For an introduction to the properties of subexponential distributions, we again refer the reader to \cite{Vershynin2011b}.}, then $X$ is rotationally invariant.
	\end{enumerate}
\end{theorem}

\begin{proof}
	Using identity \eqref{tensorid2}, we rewrite the first claim as
	\[
	\norm*{M^r_X}_2^2 \geq \norm*{M^r_{X_{rot}}}_2^2,
	\]
	and this follows immediately from equation \eqref{eq: orthogonality}.
	
	If equality holds for all positive integers $r$, then by \eqref{eq: orthogonality}, $E^r_X = 0$ for all $r$, implying that $X$ and $X_{rot}$ have the same moment tensors of all orders. Since exponential random variables are characterized by their moments (see Lemma \ref{lem: subexponential RV characterized by moments}), $X$ and $X_{rot}$ have the same distribution.
\end{proof}

\begin{proof}[Proof of Theorem \ref{thm: first gaussian test}]
	If $X$ has the same radial distribution as $g$, then $g$ is the rotational symmetrization of $X$. The claim is then a direct application of the uniqueness portion of Theorem \ref{thm: inner product moments minimized by rotational symmetrization}.
\end{proof}

We now move on to proving the robust version of the test, namely Theorem \ref{thm: first gaussian test, robust}.

\begin{lemma} \label{lem: moment deviation}
	Let $X$ be a random vector in $\R^n$. Let $\theta$ be uniformly distributed on the sphere $S^{n-1}$. Then the following hold for any positive integer $r$:
	\begin{enumerate}[nosep]
		\item[a)] $M^r_{X_{rot}} = \E{\norm{X}_2^r}M_\theta^r$.
		\item[b)] $\norm{E^r_X}^2_2 = \paren*{\E{\inprod{X,X'}}^r} - \paren*{\E{\norm{X}^r_2}}^2\paren*{\E{\inprod{\theta,\theta'}}^r}$.
		\item[c)] For any unit vector $v \in \R^n$,
		\begin{align} \label{moment deviation bound}
		\abs*{\E{\inprod{X,v}^r} - \E{\inprod{g,v}^r}} \leq \abs*{\E{\norm{X}_2^r}-\E{\norm{g}_2^r}}\E{\inprod{\theta,\theta'}^r}+\paren*{\E{\inprod{X,X'}^r} - \paren*{\E{\norm{X}^r_2}}^2\E{\inprod{\theta,\theta'}^r}}^{1/2}.
		\end{align}
		\item[d)] In particular, when $r$ is odd,
		\begin{align} \label{odd moment deviation}
		\abs*{\E{\inprod{X,v}^r} - \E{\inprod{g,v}^r}} \leq \paren*{\E{\inprod{X,X'}^r}}^{1/2} = \abs*{\E{\inprod{X,X'}^r} - \paren*{\E{\inprod{g,g'}^r}}}^{1/2}.
		\end{align}		
	\end{enumerate}
\end{lemma}

\begin{proof}
	Deferred to Appendix \ref{sec: proof of first gaussian test appendix}.
\end{proof}

By balancing the two terms on the right hand side in part c), we obtain the following lemma, whose proof is again deferred to Appendix \ref{sec: proof of first gaussian test appendix}.

\begin{lemma} \label{lem: comparison for even moments}
	Let $X$ be a random vector in $\R^n$ for $n \geq 2$. Suppose there is a unit vector $v \in S^{n-1}$, an even integer $r \geq 2$, and a positive number $0 < \delta \leq 1$ such that $\abs*{\E{\inprod{X,v}^r} - \E{ \inprod{g,v}^r} } \geq \delta\E{ \inprod{g,v}^r}$. Then either
	\[
	\abs*{\E{\norm{X}_2^r}-\E{\norm{g}_2^r}} \geq \frac{\delta^2}{4}\E{ \inprod{g,v}^r}
	\quad \text{or} \quad
	\abs*{\E{\inprod{X,X'}^r} - \E{\inprod{g,g'}^r}} \geq \frac{15\delta^2}{64} \paren*{\E{ \inprod{g,v}^r}}^2.
	\]
\end{lemma}

\begin{proof}[Proof of Theorem \ref{thm: first gaussian test, robust}]
	If $r$ is odd, then the statement follows from \eqref{odd moment deviation}. If $r$ is even, set $\delta = \frac{\Delta}{\E\inprod{g,v}^r}$ in the previous theorem.
\end{proof}

\section{Proof of the second gaussian test} \label{sec: proof of second gaussian test}

In this section, we return to the setting where $X$ follows the NGCA model \eqref{NGCA model}. We further assume that the non-gaussian component $\tilde{X}$ is a subgaussian random vector with subgaussian norm bounded by $K$. In order not to break the flow of the paper, most of the proofs are deferred to Appendix \ref{sec: proof of second gaussian test appendix}.

The first step in proving the test is to notice that the independence of the gaussian and non-gaussian components allows us to block diagonalize the test matrices.

\begin{lemma}[Block diagonalization for $\Phi_{X,\alpha}$ and $\Psi_{X,\alpha}$] \label{Phi and Psi diagonalization}
	Assume $E$ is spanned by the first $d$ basis vectors. Then the test matrices $\Phi_{X,\alpha}$ and $\Psi_{X,\alpha}$ decompose into blocks in the following manner:
	\begin{align}
	\Phi_{X,\alpha} = \left(
	\begin{array}{c|c}
	\Phi_{\tilde{X},\alpha} & 0 \\
	\hline
	0 & \Phi_{g,\alpha}
	\end{array}
	\right), \hspace{0.2in} \Psi_{X,\alpha} = \left(
	\begin{array}{c|c}
	\Psi_{\tilde{X},\alpha} & 0 \\
	\hline
	0 & \Psi_{g,\alpha}
	\end{array}
	\right).
	\end{align}
\end{lemma}

We then observe that the trace of the test matrices are conveniently equal to the negated log derivatives of their respective partition functions.

\begin{lemma}[Trace of $\Phi_{Y,\alpha}$ and $\Psi_{Y,\alpha}$] \label{lem: trace of Phi and Psi}
	Let $Y$ be any random vector in $\R^n$. Then $\tr{\Phi_{Y,\alpha}} = -(\log Z_{\Phi,Y})'(\alpha)$ and $\tr{\Psi_{Y,\alpha}} = -(\log Z_{\Psi,Y})'(\alpha)$.	
\end{lemma}

Our next lemma shows that for $\alpha$ small enough, the partition functions themselves differentiate between gaussian and non-gaussian random vectors. This is obvious once we realize that they are just the moment generating functions of $\norm{X}_2^2$ and $\inprod{X,X'}$, and that these are analytic in a small neighborhood around 0.

\begin{lemma}[Partition functions characterize gaussian distributions] \label{lem: partition functions characterize gaussians}
	The following hold for any subgaussian random vector $Y$:
	\begin{enumerate}[nosep]
		\item[a)] If $Z_{\Phi,Y}(\alpha_k) = Z_{\Phi,g}(\alpha_k)$ for a sequence of values $\alpha_k$ converging to $0$, then $Y$ has the same radial distribution as $g$.
		\item[b)] If in addition, $Z_{\Psi,Y}(\beta_k) = Z_{\Psi,g}(\beta_k)$ for a sequence of values $\beta_k$ converging to 0, then $X$ has the standard gaussian distribution.
	\end{enumerate}
\end{lemma}

We are now in a position to prove the second gaussian test.

\begin{proof}[Proof of Theorem \ref{thm: second gaussian test}]
	Let $g_d$ denote the standard gaussian in $\R^d$. By Lemma \ref{lem: partition functions characterize gaussians}, either $Z_{\Phi,\tilde{X}}(\alpha) \neq Z_{\Phi,g_d}(\alpha)$ for $\abs{\alpha}$ small enough, or $Z_{\Psi,\tilde{X}}(\alpha) \neq Z_{\Psi,g_d}(\alpha)$ for $\abs{\alpha}$ small enough. As such, either $(\log Z_{\Phi,\tilde{X}})'(\alpha) \neq (\log Z_{\Phi,g_d})'(\alpha)$ or $(\log Z_{\Psi,\tilde{X}})'(\alpha) \neq (\log Z_{\Psi,g_d})'(\alpha)$. Assume the former holds, and let $\lambda_1,\ldots, \lambda_n$ denote the eigenvalues of $\Phi_{X,\alpha}$. Since we may write $\Phi_{X,\alpha}$ in a block form, these eigenvalues are either those of $\Phi_{\tilde{X},\alpha}$ or $\Phi_{g,\alpha}$. Without loss of generality, we may assume that $\lambda_1,\ldots,\lambda_d$ are the eigenvalues of $\Phi_{\tilde{X},\alpha}$, and $\lambda_{d+1},\ldots,\lambda_n$ are those of $\Phi_{g,\alpha}$.
	
	Lemma \ref{lem: formula for Phi and Psi for gaussian} tells us that $\lambda_{d+1} = \cdots = \lambda_n = (2\alpha+1)^{-1}$. On the other hand, by Lemma \ref{lem: trace of Phi and Psi},
	\[
	\sum_{i=1}^d\lambda_i = \tr{\Phi_{\tilde{X},\alpha}} = -(\log Z_{\Phi,\tilde{X}})'(\alpha).
	\]
	By Lemma \ref{lem: formula for log derivatives for gaussian}, $-(\log Z_{\Phi,g_d})'(\alpha) = d(2\alpha+1)^{-1}$, so we have $\sum_{i=1}^d\lambda_i \neq d(2\alpha+1)^{-1}$. Dividing through by $d$, we get $\frac{1}{d}\sum_{i=1}^d\lambda_i \neq (2\alpha+1)^{-1}$, which implies that at least one $\lambda_i$ differs from this value for $1 \leq i \leq d$.
	
	If it were the case that $(\log Z_{\Psi,\tilde{X}})'(\alpha) \neq (\log Z_{\Psi,g_d})'(\alpha)$, a similar argument involving $\Psi_{X,\alpha}$ gives the alternate conclusion.
\end{proof}

It is tedious but not too difficult to make the second gaussian test quantitative. We do this by tracking how the non-gaussian moments for $\norm{\tilde{X}}_2$ and $\inprod{\tilde{X},\tilde{X}'}$ contribute to the power series expansions for $-(\log Z_{\Phi,\tilde{X}})'$ and $-(\log Z_{\Psi,\tilde{X}})'$ around 0. This yields the following theorem.

\begin{theorem}[Second gaussian test, robust] \label{thm: second gaussian test, robust}
	Let $r$ be the integer such that $D_{\tilde{X},r} > 0$ and $D_{\tilde{X},r'} = 0$ for all $r' < r$. Then either
	\begin{enumerate}
		\item[a)] for $\abs{\alpha} \leq \eta_r^2 r/(CK^2)^r(d^{r+1}+(r+1)!)$, we have
		\begin{align}
		\abs*{\frac{1}{d}\sum_{i=1}^d\lambda_i(\Psi_{\tilde{X},\alpha})-\frac{\alpha}{\alpha^2-1}} \geq \frac{c\eta_r^2}{d(r-1)!}\abs{\alpha}^{r-1},
		\end{align}
		\item[b)] or for $\abs{\alpha} \leq \eta_r^2 r/(CK^2)^{r/2}\gamma_r(d^{r/2+1}+(r/2+1)!)$, we have
		\begin{align}
		\abs*{\frac{1}{d}\sum_{i=1}^d\lambda_i(\Phi_{\tilde{X},\alpha})-\frac{1}{2\alpha+1}} \geq \frac{c\eta_r^2}{d(r/2-1)!\gamma_r}\abs{\alpha}^{r/2-1}.
		\end{align}
		Here $\gamma_r = \E|\inprod{g,v}|^r$ for an arbitrary vector $v \in S^{n-1}$ 
		and $\eta_r = \min\braces{D_{X,r}, \gamma_r}$.  
	\end{enumerate}
\end{theorem}

\section{Proof of guarantee for Reweighted PCA} \label{sec: proof of guarantee}

The second gaussian test tells us how we can recover non-gaussian directions from $\Phi_{X,\alpha}$ and $\Psi_{X,\alpha}$. Our guarantee for Reweighted PCA algorithm shows that we can do the same with the plug-in estimators $\hat{\Phi}_{X,\alpha}$ and $\hat{\Psi}_{X,\alpha}$. To this end, we first provide concentration bounds for these estimators, whose proofs can be found in Appendix \ref{concentration of estimators}.

\begin{theorem}[Concentration for $\hat{\Phi}_{X,\alpha}$] \label{Phi concentration}
	There is an absolute constant $C$ such that for any $0 < \epsilon, \delta < 1$, and any $0 \leq \alpha < 1/CK^2n$, we have $\Prob[]{\norm{\hat{\Phi}_{X,\alpha} - \Phi_{X,\alpha}} > \epsilon} \leq \delta$ so long as $N \geq CK^2(n+\log(1/\delta))\epsilon^{-2}$.
\end{theorem}

\begin{theorem}[Concentration for $\hat{\Psi}_{X,\alpha}$] \label{Psi concentration}
	There is an absolute constant $C$ such that for any $0 < \epsilon, \delta < 1$, if $N \geq CK^2(n+\log(1/\delta))\epsilon^{-2}$ and $\abs{\alpha} \leq 1/CK^2\tau(n+\tau)$, we have $\Prob[]{\norm{\hat{\Psi}_{X,\alpha} - \Psi_{X,\alpha}} > \epsilon} \leq \delta$. Here, $\tau = \log^{1/2}(N/\min\braces{\delta,K\epsilon})$.
\end{theorem}

We use the following notion of distance between subspaces.

\begin{definition}[Subspace distance]
	Let $F$ and $F'$ be subspaces of $\R^n$ of dimensions $m$. Let $U$ $U'$ be matrices whose columns form an orthonormal basis for $F$ and $F'$ respectively. The distance between $F$ and $F'$ is defined to be $d(F,F') := \norm{UU^T-U'(U')^T}_F$.
\end{definition}

\begin{lemma}[Guarantee for $\hat{E}_\Phi$]
	Suppose the moments of $\norm{\tilde{X}}_2^2$ and $\norm{g_d}_2^2$ agree up to order $r-1$, but there is a number $\Delta > 0$ such that $\abs*{\E{ \norm{\tilde{X}}_2^{2r}} - \E{ \norm{g_d}_2^{2r}}} \geq \Delta$. For any $\delta,\epsilon \in (0,1)$, pick $\alpha_1$ such that  $0 < \alpha_1 < \min\braces{\Delta r/(CK^2)^r(d^{r+1}+(r+1)!),1/CK^2n}$, and $\beta_1 = \Delta\alpha_1^{r-1}/4d(r-1)!$. Then with probability at least $1-\delta$, {\sc Reweighted PCA} with $2N \geq CK^2d^{3/2}(n+\log(1/\delta))/\beta_1^2\epsilon^2$ samples together with this choice of $\alpha_1$ and $\beta_1$ produces a nontrivial estimate $\hat{E}_\Phi$ of dimension $1 \leq \hat{d}_\Phi \leq d$, such that there is a $\hat{d}_\Phi$-dimensional subspace $E_\Phi \subset E$ satisfying $d(\hat{E}_\Phi,E_\Phi) \leq \epsilon$.
\end{lemma}

\begin{proof}
	Combining Lemmas \ref{Phi and Psi diagonalization}, \ref{lem: trace of Phi and Psi}, and \ref{lem: formula for Phi and Psi for gaussian} tells us that in the right coordinates, $\Phi_{X,\alpha}$ block diagonalizes as
	\begin{align}
	\Phi_{X,\alpha} = \left(
	\begin{array}{c|c}
	\Phi_{\tilde{X},\alpha} & 0 \\
	\hline
	0 & (2\alpha+1)^{-1}I_{n-d},
	\end{array}
	\right).
	\end{align}
	
	Next, label the eigenvalues of $\Phi_{X,\alpha_1}$ as $\lambda_1 \geq \lambda_2 \geq \cdots \geq \lambda_n$. We can find $0 \leq p \leq q \leq n$ such that the eigenvalues corresponding to the $\Phi_{\tilde{X},\alpha_1}$ block are $\lambda_1,\lambda_2,\ldots,\lambda_p,\lambda_{q+1},\ldots,\lambda_n$. Using Theorem \ref{Phi and Psi difference theorem}, we then have
	\begin{equation}
	\abs*{\frac{1}{d}\paren*{\sum_{i=1}^p\lambda_i+\sum_{i=q+1}^n\lambda_i}-\frac{1}{2\alpha_1+1}} \geq \frac{\Delta}{2d(r-1)!}\alpha_1^{r-1} = 2\beta_1.
	\end{equation}
	
	In particular, we have $\frac{1}{p}\sum_{i=1}^p\lambda_i - 1/(2\alpha_1+1) \geq 2\beta_1$, and $1/(2\alpha_1+1) - \frac{1}{n-q}\sum_{i=q+1}^n\lambda_i \geq 2\beta_1$. Since at least one of these sums of eigenvalues is non-empty, truncating the eigenvalues of $\Phi_{X,\alpha_1}$ at the $\beta_1$ level gives us a non-trivial subspace of $E$.
	
	In order to show that our empirical estimate $\hat{\Phi}_{X,\alpha_1}$ also has an approximation to this property, we will need to use the eigenvector perturbation theory explained in Appendix \ref{eigenvector perturbation}. First, we need to bound from below the ``eigengap'' in $\Phi_{X,\alpha_1}$. Suppose first that $p \geq 1$, i.e. that there are eigenvalues larger than $(2\alpha_1+1)^{-1}$. Then by the pigeonhole principle, one can find $i$ such that $(2\alpha_1+1)^{-1} + \beta_1/2 \geq \lambda_{i+1} \geq (2\alpha_1+1)^{-1}$ and $\lambda_i - \lambda_{i+1} \geq \beta_1/2d$. Similarly, if $q \leq n-1$, then we can find $j$ such that $(2\alpha_1+1)^{-1} \geq \lambda_{j-1} \geq (2\alpha_1+1)^{-1} - \beta_1/2$ and $\lambda_{j-1}-\lambda_j \geq \beta_1/2d$.
	
	Now let $F$ be the span of the eigenvectors of $\Phi_{X,\alpha_1}$ corresponding to $\lambda_1,\ldots,\lambda_i,\lambda_j,\ldots,\lambda_n$, and let $\hat{F}$ be the eigenvectors of $\hat{\Phi}_{X,\alpha_1}$ corresponding to $\hat{\lambda}_1,\ldots,\hat{\lambda}_i,\hat{\lambda}_j,\ldots,\hat{\lambda}_n$. By Theorem \ref{Phi concentration}, with probability at least $1-\delta$, we have
	\begin{equation} \label{Phi difference for correctness proof}
	\norm{\hat{\Phi}_{X,\alpha} - \Phi_{X,\alpha}} \leq \frac{\beta_1\epsilon}{4\sqrt{2}d^{3/2}}.
	\end{equation}
	We may then use Theorem \ref{Davis-Kahan theorem} to see that $d(\hat{F},F) \leq \epsilon$.
	
	We are not yet done, because we do not have access to $\hat{F}$. Nonetheless, we can show that $\hat{F}$ contains $\hat{E}_\Phi$. Using eigenvalue perturbation inequalities together with equation \eqref{Phi difference for correctness proof} tells us that we have
	\begin{equation}
	\hat{\lambda}_{i+1} \leq \lambda_{i+1} + \frac{\beta_1\epsilon}{2d} \leq (2\alpha_1+1)^{-1} + \frac{\beta_1}{2} + \frac{\beta_1\epsilon}{2d} \leq (2\alpha_1+1)^{-1} + \frac{2\beta_1}{3},
	\end{equation}
	and similarly that
	\begin{equation}
	\hat{\lambda}_{j-1} \leq \lambda_{j-1} - \frac{\beta_1\epsilon}{2d} \leq (2\alpha_1+1)^{-1} - \frac{\beta_1}{2} - \frac{\beta_1\epsilon}{2d} \leq (2\alpha_1+1)^{-1} - \frac{2\beta_1}{3}.
	\end{equation}
	Let $\hat{I}_\Phi = \braces{i \ \colon \abs{\hat{\lambda}_i - (1-2\alpha_1)^{-1}} > \beta_1}$. We see that this set does not contain any index between $i+1$ and $j-1$, so $\hat{E}_\Phi$, which comprises the span of the eigenvectors to these eigenvalues, does not contain any eigenvector that $\hat{F}$ does not contain, as was to be shown. The inclusion then implies that we may find a subspace $E_\Phi \subset F$ such that $d(\hat{E}_\Phi,E_\Phi) \leq \epsilon$.
	
	Finally, we observe that $\dim{\hat{E}_\Phi} \geq 1$, since
	\begin{equation}
	\frac{1}{p}\sum_{i=1}^p\hat{\lambda}_i - \frac{1}{2\alpha_1+1} \geq \frac{1}{p}\sum_{i=1}^p\lambda_i - \frac{\beta_1\epsilon}{2d} - \frac{1}{2\alpha_1+1} > \beta_1,
	\end{equation}
	and
	\begin{equation}
	\frac{1}{2\alpha_1+1} - \frac{1}{n-q}\sum_{i=q+1}^n\hat{\lambda}_i \geq \frac{1}{2\alpha_1+1} - \frac{1}{n-q}\sum_{i=q+1}^n\lambda_i - \frac{\beta_1\epsilon}{2d} > \beta_1.
	\end{equation}	
\end{proof}

\begin{lemma}[Guarantee for $\hat{E}_\Psi$]
	Suppose the moments of $\inprod{X,X'}$ and $\inprod{g,g'}$ agree up to order $r-1$ but $\abs*{\E{ \inprod{X,X'}^r} - \E{ \inprod{g,g'}^r}} \geq \Delta$. For any $\delta,\epsilon,\tau \in (0,1)$, pick $0 < \alpha_2 < \min\braces{\Delta r/(CK^2)^r(d^{r+1}+(r+1)!),1/CK^2n^{1+\tau}}$, and $\beta_2 = \Delta\alpha_2^{r-1}/4d(r-1)!$. Then with probability at least $1-\delta$, {\sc Reweighted PCA} with sample size $2N$ satisfying $\exp(n^{2\tau})\min\braces{\delta,K\epsilon} \geq 2N \geq CK^2d^{3/2}(n+\log(1/\delta))/\beta_2^2\epsilon^2$, together with this choice of $\alpha_2$ and $\beta_2$ produces a nontrivial estimate $\hat{E}_\Psi$ of dimension $1 \leq \hat{d}_\Psi \leq d$, such that there is a $\hat{d}_\Psi$-dimensional subspace $E_\Psi \subset E$ satisfying $d(\hat{E}_\Psi,E_\Psi) \leq \epsilon$.
\end{lemma}

\begin{proof}
	The proof is completely analogous to that for the previous theorem, except that we replace our estimates and identities for $\Phi_{X,\alpha_1}$ with those for $\Psi_{X,\alpha_2}$ wherever necessary.
\end{proof}

\begin{proof}[Proof of Theorem \ref{PCA guarantee}]
	Combine the last two lemmas with Theorem \ref{thm: second gaussian test, robust} from the last section.
\end{proof}

\begin{remark}[Selecting optimal parameters]
	If the problem parameters $d, D, K$ and $r$ were known before hand, then in principle, one could compute the optimal tuning parameters $\alpha_1,\alpha_2,\beta_1,\beta_2$. In practice, however, one rarely is in this situation, so one would have to estimate the problem parameters as a first step to solving the NGCA problem. Nonetheless, one can do this by the doubling/halving trick. In other words, we start with some fixed initial choice of $\alpha_1$ and $\alpha_2$. Using Theorems \ref{Phi concentration} and \ref{Psi concentration}, we can detect whether there are any outlier eigenvalues with high probability. If there are none, we halve $\alpha_1$ and $\alpha_2$ and try again, repeating this process until outliers show up. The number of iterations is then the base 2 logarithm of the final $\alpha_1$ and $\alpha_2$, plus an additive constant. This is logarithmic in $d, D$, and $K$, and polynomial in $r$, so the algorithm remains efficient.
\end{remark}

\section{Discussion} \label{discussion}

We have presented and analyzed an algorithm that is guaranteed to return at least one non-gaussian direction efficiently, with sample and time complexity a polynomial in the problem parameters for a fixed $r$, where $r$ is the smallest order at which the moments of any marginal of $\tilde{X}$ differ from those of a standard gaussian.

Since the degree of the polynomial increases linearly in $r$, it would seem that the algorithm is practically useless if $r$ is larger than a small constant. However, note that having all third and fourth moments equal those of a gaussian is a condition that is already stringent in one dimension, and which becomes even more so in higher dimensions. As such, unless $\tilde{X}$ has some kind of adversarial distribution, $r$ will be either $4$ or $3$, depending on whether $\tilde{X}$ is centrally symmetric or not.

The algorithm also often delivers much more than is guaranteed for several reasons. First, in order to bound the subspace perturbation by $\epsilon$, we used a very crude estimate of the eigengap, bounding it from below using the pigeonhole principle, which in the worst case assumes that the eigenvalues are spread out at regular intervals. This should not happen in practice, and we expect the non-gaussian eigenvalues to instead cluster relatively tightly around their average. If this happens, the sample complexity requirement can be relaxed by a factor of $d$.

Second, just as it is extremely unlikely for $r$ to be higher than $4$, for a general non-gaussian $\tilde{X}$ and a small, random $\alpha$, it is extremely unlikely for any of the non-gaussian values of $\Phi_{X,\alpha}$ to be equal to the gaussian one on the dot. This means that even though the guarantee is for one direction, in practice we most probably can recover the entire subspace $E$ simultaneously with just $\hat{\Phi}_{X,\alpha}$ alone, albeit with a more sophisticated truncation technique. Indeed, we can provably recover all non-gaussian directions when $\tilde{X}$ is an unconditional random vector, examples of which include the uniform distributions on spheres and hypercubes.

\subsection{Conjectures and questions}

We conjecture that Reweighted PCA actually recovers the entire non-gaussian subspace $E$, so long as $\tilde{X}$ does not have any gaussian marginals. The first gaussian test for a random vector $X$ using the distribution of its norm and dot product pairing also leads to further questions. For a fixed nonzero real number $t$, both of these appear in the formula for $\norm{Y_t}_2^2$, where we set $Y_t := X+tX'$, so it is natural to ask whether Reweighted PCA works with $\Phi_{Y_t,\alpha}$ alone for some $t$. In particular, does it work for $t=-1$? It is also an open question whether $\inprod{X,X'}$ alone is sufficient to test whether $X$ is standard gaussian.

\subsection*{Acknowledgements}

Both authors are partially supported by NSF Grant DMS 1265782 and U.S. Air Force Grant FA9550-14-1-0009.

\nocite{*}
\bibliographystyle{acm}
\bibliography{NGCA}

\begin{thebibliography}{10}

\bibitem{Arora2015}
{\sc Arora, S., Ge, R., Moitra, A., and Sachdeva, S.}
\newblock {Provable ICA with Unknown Gaussian Noise, and Implications for
  Gaussian Mixtures and Autoencoders}.
\newblock {\em Algorithmica 72}, 1 (2015), 215--236.

\bibitem{Billingsley1995b}
{\sc Billingsley, P.}
\newblock {\em {Probability and Measure - Third Edition}}.
\newblock 1995.

\bibitem{Blanchard2006b}
{\sc Blanchard, G., Kawanabe, M., Sugiyama, M., Spokoiny, V., and M{\"{u}}ller,
  K.-R.}
\newblock {In Search of Non-Gaussian Components of a High-Dimensional
  Distribution}.
\newblock {\em Journal of Machine Learning Research 7}, 2 (2006), 247--282.

\bibitem{Brubaker2008b}
{\sc Brubaker, S.~C., and Vempala, S.~S.}
\newblock {Isotropic PCA and affine-invariant clustering}.
\newblock {\em Proceedings - Annual IEEE Symposium on Foundations of Computer
  Science, FOCS\/} (2008), 551--560.

\bibitem{Cnlar2011b}
{\sc {\c{C}}inlar, E.}
\newblock {\em {Probability and stochastics}}, vol.~230.
\newblock 2011.

\bibitem{Comon1994}
{\sc Comon, P.}
\newblock {Independent component analysis, A new concept?}
\newblock {\em Signal Processing 36}, 3 (1994), 287--314.

\bibitem{Davis1970a}
{\sc Davis, C., and Kahan, W.~M.}
\newblock {The rotation of eigenvectors by a perturbation. III}.
\newblock {\em SIAM Journal on Numerical Analysis 7}, 1 (1970), 1--46.

\bibitem{Diederichs2013b}
{\sc Diederichs, E., Juditsky, A., Nemirovski, A., and Spokoiny, V.}
\newblock {Sparse non Gaussian component analysis by semidefinite programming}.
\newblock {\em Machine Learning 91}, 2 (2013), 211--238.

\bibitem{Diederichs2010b}
{\sc Diederichs, E., Juditsky, A., Spokoiny, V., and Sch{\"{u}}tte, C.}
\newblock {Sparse non-Gaussian component analysis}.
\newblock {\em IEEE Transactions on Information Theory 56}, 6 (2010),
  3033--3047.

\bibitem{Frieze1996}
{\sc Frieze, A., Jerrum, M., and Kannan, R.}
\newblock {Learning linear transformations}.
\newblock {\em Proceedings of 37th Conference on Foundations of Computer
  Science\/} (1996), 359--368.

\bibitem{Goyal2014a}
{\sc Goyal, N., Vempala, S., and Xiao, Y.}
\newblock {Fourier PCA and robust tensor decomposition}.
\newblock In {\em Proceedings of the 46th Annual ACM Symposium on Theory of
  Computing - STOC '14\/} (New York, New York, USA, 2014), no.~c, ACM Press,
  pp.~584--593.

\bibitem{Huber1985b}
{\sc Huber, P.~J.}
\newblock {Projection Pursuit}.
\newblock {\em The Annals of Statistics 13}, 2 (1985), 435--475.

\bibitem{Hyvarinen1997}
{\sc Hyvarinen, A.}
\newblock {Fast and robust fixed-point algorithms for independent component
  analysis}.
\newblock {\em IEEE Transactions on Neural Networks 10}, 3 (May 1999),
  626--634.

\bibitem{Kawanabe2006}
{\sc Kawanabe, M., Sugiyama, M., Blanchard, G., and M{\"{u}}ller, K.-R.}
\newblock {A new algorithm of non-Gaussian component analysis with radial
  kernel functions}.
\newblock {\em Annals of the Institute of Statistical Mathematics 59}, 1
  (2006), 57--75.

\bibitem{Kawanabe2006b}
{\sc Kawanabe, M., and Theis, F.~J.}
\newblock {Estimating non-Gaussian subspaces by characteristic functions}.
\newblock {\em Lecture Notes in Computer Science (including subseries Lecture
  Notes in Artificial Intelligence and Lecture Notes in Bioinformatics) 3889
  LNCS\/} (2006), 157--164.

\bibitem{Kawanabe2007b}
{\sc Kawanabe, M., and Theis, F.~J.}
\newblock {Joint low-rank approximation for extracting non-Gaussian subspaces}.
\newblock {\em Signal Processing 87}, 8 (2007), 1890--1903.

\bibitem{Kawanabe2005b}
{\sc {M. Kawanabe}}.
\newblock {Linear dimension reduction based on the fourth-order cumulant
  tensor}.
\newblock {\em Proc. of Artifical Neural Networks -- ICANN 2005\/} (2005),
  151--156.

\bibitem{Sasaki2014}
{\sc Sasaki, H., Hyv{\"{a}}rinen, A., and Sugiyama, M.}
\newblock {Clustering via mode seeking by direct estimation of the gradient of
  a log-density}.
\newblock {\em Lecture Notes in Computer Science (including subseries Lecture
  Notes in Artificial Intelligence and Lecture Notes in Bioinformatics) 8726
  LNAI}, PART 3 (2014), 19--34.

\bibitem{Sasaki2016b}
{\sc Sasaki, H., Niu, G., and Sugiyama, M.}
\newblock {Non-Gaussian Component Analysis with Log-Density Gradient
  Estimation}.
\newblock {\em International Conference on Artificial Intelligence and
  Statistics 51\/} (2016), 1--20.

\bibitem{Tan2016}
{\sc Tan, Y.~S.}
\newblock {Energy optimization for distributions on the sphere and improvement
  to the Welch bounds}.
\newblock Available at \url{http://arxiv.org/abs/1612.06343}.

\bibitem{Vershynin2011b}
{\sc Vershynin, R.}
\newblock {Introduction to the non-asymptotic analysis of random matrices}.
\newblock In {\em Compressed Sensing}, Y.~C. Eldar and G.~Kutyniok, Eds.
  Cambridge University Press, Cambridge, 2011, pp.~210--268.

\bibitem{Virta2016b}
{\sc Virta, J., Nordhausen, K., and Oja, H.}
\newblock {Projection Pursuit for non-Gaussian Independent Components}.
\newblock Available at \url{https://arxiv.org/abs/1612.05445}.

\bibitem{Yu2015}
{\sc Yu, Y., Wang, T., and Samworth, R.~J.}
\newblock {A useful variant of the Davis-Kahan theorem for statisticians}.
\newblock {\em Biometrika 102}, 2 (2015), 315--323.

\end{thebibliography}

\appendix

\section{Details for Section \ref{sec: proof of first gaussian test}} \label{sec: proof of first gaussian test appendix}

\begin{lemma} \label{lem: subexponential RV characterized by moments}
	Let $X$ be a subexponential random vector in $\R^n$. Then the distribution of $X$ is determined by its moment tensors.
\end{lemma}

\begin{proof}
	Let $\phi_X(v) = \E{e^{i\inprod{X,v}}}$ denote the characteristic function of $X$, and let $K = \norm{X}_{\psi_1}$ denote the subexponential norm of $X$. We then have the following moment growth condition \cite{Vershynin2011b}:
	\begin{align} \label{moments}
	\sup_{v \in S^{n-1}}\limsup_{r \to \infty} \frac{\paren*{\E{\abs*{\inprod{X,v}}^r}}^{1/r}}{r} \lesssim K.
	\end{align}
	This condition implies that for each $v \in S^{n-1}$, the function $t \mapsto \E{e^{it\inprod{X,v}}}$ can be written as a power series with coefficients $\frac{\E{\inprod{X,v}^r}}{r!}$ \cite{Billingsley1995b}, so $\phi_X(v)$ is determined by the moments $\E{\inprod{X,v}^r}$. By \eqref{tensorid1}, $\E{\inprod{X,v}^r} = \inprod{M_X^k,v^{\tensor r}}$, so these are functions of the moment tensors. Finally, it is a fact from elementary probability that $X$ is determined by its characteristic function \cite{Cnlar2011b}.
\end{proof}

\begin{proof}[Proof of Lemma \ref{lem: moment deviation}]
	For the first statement, observe that $X_{rot} = \norm{X}_2\theta$, with $\norm{X}_2$ and $\theta$ independent. We thus have
	\[
	M^r_{X_{rot}} = \E{\paren*{\norm{X}_2\theta}^{\tensor r}} = \E{\norm{X}_2^r}\E{\theta^{\tensor r}} =  \E{\norm{X}_2^r}M_\theta^r.
	\]
	Next, rewrite \eqref{eq: orthogonality} as $\norm{E^r_X}^2_2 = \norm{M^r_X}_2^2 - \norm{M^r_{X_{rot}}}_2^2$. By definition, we have $\norm{M^r_X}_2^2 = \E{\inprod{X,X'}^r}$ and using a), we get $\norm{M^r_{X_{rot}}}_2^2 = \paren*{\E{\norm{X}^r_2}}^2\E{\inprod{\theta,\theta'}^r}$.
	
	To prove part c), fix $v$ and write
	\begin{align*}
	\E{\inprod{X,v}^r} - \E{\inprod{g,v}^r} & = \inprod{M^r_X-M^r_g,v^{\tensor r}} = \inprod{M^r_{X_{rot}}-M^r_g,v^{\tensor r}} + \inprod{E^r_X,v^{\tensor r}}.
	\end{align*}
	We use a) to write
	\[
	\inprod{M^r_{X_{rot}}-M^r_g,v^{\tensor r}} = \inprod{\E{\norm{X}_2^r}M^r_\theta-\E{\norm{g}_2^r}M^r_\theta,v^{\tensor r}} = \paren*{\E{\norm{X}_2^r}-\E{\norm{g}_2^r}}\E{\inprod{\theta,v}^r}.
	\]
	Notice that $\E{\inprod{\theta,v}^r} = \E{\inprod{\theta,\theta'}^r}$. We then combine the last two equations with b) and Cauchy-Schwarz to get \eqref{moment deviation bound}. Finally, to get the last claim, we use the fact that $\E{\inprod{\theta,\theta'}^r} = \E{\inprod{g,g'}^r} = 0$ whenever $r$ is odd.
\end{proof}

\begin{proof}[Proof of Lemma \ref{lem: comparison for even moments}]
	Observe that \eqref{moment deviation bound} gives the bound
	\begin{align} \label{marginal moment 1}
	\delta\E{ \inprod{g,v}^r} \leq \abs*{\E{\norm{X}_2^r}-\E{\norm{g}_2^r}}\E{\inprod{\theta,\theta'}^r}+\paren*{\E{\inprod{X,X'}^r} - \paren*{\E{\norm{X}^r_2}}^2\paren*{\E{\inprod{\theta,\theta'}^r}}}^{1/2}.
	\end{align}
	Suppose $\abs*{\E{\norm{X}_2^r}-\E{\norm{g}_2^r}} \leq \frac{\delta^2}{4}\E{ \inprod{g,v}^r}$. Then the second term on the right in equation \eqref{marginal moment 1} has to be large. Indeed, since $\delta \leq 1$ and $\E{\inprod{\theta,\theta'}^r} \leq 1/2$ for $r,n \geq 2$, we have
	\begin{align*}
	\paren*{\E{\inprod{X,X'}^r} - \paren*{\E{\norm{X}^r_2}}^2\E{\inprod{\theta,\theta'}^r}}^{1/2} & \geq \delta \E{ \inprod{g,v}^r} - \frac{\delta^2}{4}\E{\inprod{\theta,\theta'}^r}\E{ \inprod{g,v}^r} \\
	& \geq \frac{7\delta}{8} \E{ \inprod{g,v}^r}.
	\end{align*}	
	Now, applying the fact that $\E{ \inprod{g,g'}^r} = \paren*{\E{\norm{g}_2^r}}^2\E{ \inprod{\theta,\theta'}^r}$, we use the reverse triangle inequality and the above bound to write
	\begin{align} \label{marginal moment 2}
	\abs*{\E{\inprod{X,X'}^r} - \E{\inprod{g,g'}^r}} & \geq \abs*{\E{\inprod{X,X'}^r} - \paren*{\E{\norm{X}^r_2}}^2\E{\inprod{\theta,\theta'}^r}} - \abs*{\paren*{\E{\norm{X}_2^r}}^2-\paren*{\E{\norm{g}_2^r}}^2}\E{\inprod{\theta,\theta'}^r} \nonumber \\
	& \geq \paren*{\frac{7\delta}{8} \E{ \inprod{g,v}^r}}^2 - \abs*{\paren*{\E{\norm{X}_2^r}}^2-\paren*{\E{\norm{g}_2^r}}^2}\E{\inprod{\theta,\theta'}^r}.
	\end{align}
	
	Next, notice that
	\begin{align*}
	\abs*{\paren*{\E{\norm{X}_2^r}}^2-\paren*{\E{\norm{g}_2^r}}^2} & = \abs*{\E{\norm{X}_2^r}-\E{\norm{g}_2^r}} \paren*{\E{ \norm{X}_2^r} +\E{ \norm{g}_2^r}} \\
	& = \abs*{\E{\norm{X}_2^r}-\E{\norm{g}_2^r}}\cdot 2\E{ \norm{g}_2^r} + \paren*{\E{\norm{X}_2^r}-\E{\norm{g}_2^r}}^2,
	\end{align*}
	so by the assumption on $\abs*{\E{\norm{X}_2^r}-\E{\norm{g}_2^r}}$, we have
	\begin{align} \label{marginal moment 3}
	\abs*{\paren*{\E{\norm{X}_2^r}}^2-\paren*{\E{\norm{g}_2^r}}^2}\E{\inprod{\theta,\theta'}^r} & \leq \frac{\delta^2}{4}\E{ \inprod{g,v}^r}\cdot2\E{ \norm{g}_2^r}\cdot\E{\inprod{\theta,\theta'}^r} + \paren*{\frac{\delta^2}{4}\E{ \inprod{g,v}^r}}^2\E{\inprod{\theta,\theta'}^r} \nonumber \\
	& = \frac{\delta^2}{2} \paren*{\E{ \inprod{g,v}^r}}^2 + \paren*{\frac{\delta^2}{4}\E{ \inprod{g,v}^r}}^2\E{\inprod{\theta,\theta'}^r} \nonumber \\
	& \leq \frac{17\delta^2}{32}\paren*{\E{ \inprod{g,v}^r}}^2.
	\end{align}
	We can now substitute \eqref{marginal moment 3} into \eqref{marginal moment 2} to get
	\begin{equation*}
	\abs*{\E{\inprod{X,X'}^r} - \E{\inprod{g,g'}^r}} \geq \frac{15\delta^2}{64}\paren*{\E{ \inprod{g,v}^r}}^2. \qedhere
	\end{equation*}
\end{proof}

\section{Details for Section \ref{sec: proof of second gaussian test}} \label{sec: proof of second gaussian test appendix}

\begin{proof}[Proof of Lemma \ref{Phi and Psi diagonalization}]
	The decompositions follow easily from the independence of the two components of the mixed vector, $\tilde{X}$ and $g$, as well as the unconditional symmetry of the gaussian component. Let us illustrate this by proving the decomposition for $\Phi_{X,\alpha}$. First, note that $e^{-\alpha\norm{X}_2^2} = e^{-\alpha\norm{\tilde{X}}_2^2}e^{-\alpha\norm{g}_2^2}$, so that $Z_{\Phi,X}(\alpha) = Z_{\Phi,\tilde{X}}(\alpha)Z_{\Phi,g}(\alpha)$. The top left $d$ by $d$ block is hence given by
	\[
	\frac{\E{e^{-\alpha\norm{X}_2^2}\tilde{X}\tilde{X}^T}}{Z_{\Phi,X}(\alpha)} = \frac{Z_{\Phi,g}(\alpha)\E{e^{-\alpha\norm{\tilde{X}}_2^2}\tilde{X}\tilde{X}^T}}{Z_{\Phi,X}(\alpha)} = \frac{\E{e^{-\alpha\norm{\tilde{X}}_2^2}\tilde{X}\tilde{X}^T}}{Z_{\Phi,\tilde{X}}(\alpha)} = \Phi_{\tilde{X},\alpha}.
	\]
	The bottom right $d'$ by $d'$ block is also computed similarly. Finally, any entry outside these two blocks is of the form
	\begin{equation*}
	\frac{\E{e^{-\alpha\norm{X}_2^2}\tilde{X}_ig_j}}{Z_{\Phi,X}(\alpha)} = \frac{\E{e^{-\alpha\norm{X}_2^2}\tilde{X}_i(-g_j)}}{Z_{\Phi,X}(\alpha)} = -\frac{\E{e^{-\alpha\norm{X}_2^2}\tilde{X}_ig_j}}{Z_{\Phi,X}(\alpha)} = 0. \qedhere
	\end{equation*}
\end{proof}

\begin{proof}[Proof of Lemma \ref{lem: trace of Phi and Psi}]
	We have
	$$
	\tr{\Phi_{X,\alpha}} = \frac{\E\norm{X}_2^2e^{-\alpha\norm{X}_2^2}}{\E e^{-\alpha\norm{X}_2^2}} = \frac{-Z_{\Phi,X}'(\alpha)}{Z_{\Phi,X}(\alpha)} = -(\log Z_{\Phi,X})'(\alpha).
	$$
	The calculation for $\Psi_{X,\alpha}$ is similar.
\end{proof}

In order to prove Lemma \ref{lem: partition functions characterize gaussians}, we first need to establish the analyticity for the two partition functions.

\begin{lemma}[Analyticity for $Z_{\Phi,X}$ and $Z_{\Psi,X}$] \label{lem: analyticity}
	Let $X$ be a subgaussian random vector in $\R^n$ with subgaussian norm bounded by $K \geq 1$. The functions $Z_{\Phi,X}$ and $Z_{\Psi,X}$ are both analytic on $(-1/CK^2,1/CK^2)$. They are given by the formulae $Z_{\Phi,X}(\alpha) = \sum_{r=0}^\infty\E{\norm{X}^{2r}_2}(-\alpha)^r/r!$ and $Z_{\Psi,X}(\alpha) = \sum_{r=0}^\infty\E{\inprod{X,X'}^r}(-\alpha)^r/r!$. Furthermore, by choosing $C$ sufficiently large, on this interval they satisfy the bounds
	\begin{align} \label{MGF bounds}
	\abs{Z_{\Phi,X}(\alpha)}, \abs{Z_{\Psi,X}(\alpha)} \leq e^{CK^2n\abs{\alpha}} + \frac{CK^2\abs{\alpha}}{1-CK^2\abs{\alpha}}.
	\end{align}
\end{lemma}

\begin{proof}
	Let us first prove the bounds in \eqref{MGF bounds}. Observe that
	\begin{align} \label{MGF expansion}
	\E{e^{-\alpha\norm{X}_2^2}} \leq \E{e^{\abs{\alpha}\norm{X}_2^2}} = \sum_{n=0}^\infty \frac{\E{\norm{X}_2^{2n}}}{n!}\abs{\alpha}^n.
	\end{align}
	Here, Tonelli allows us to interchange the sum and expectation. We next use Lemma \ref{norm concentration} to bound the terms of this series. Indeed, using the equivalent estimate \eqref{alternate norm moment bound}, we have
	\[
	\E{\norm{X}_2^{2r}} \leq C^rK^{2r}\paren{n^r+r!}
	\]
	for some universal constant $C$. Substituting this into \eqref{MGF expansion} and using $\abs{\alpha} \leq 1/CK^2$, we have
	\begin{align*}
	\E{e^{-\alpha\norm{X}_2^2}} & \leq \sum_{r=0}^\infty \frac{(CK^2)^r(n^r+r!)}{r!}\abs{\alpha}^r \\
	& = \sum_{r=0}^\infty \frac{\paren{CK^2n\abs{\alpha}}^r}{r!} + \sum_{r=1}^\infty \paren{CK^2\abs{\alpha}}^r \\
	& = e^{CK^2n\abs{\alpha}} + \frac{CK^2\abs{\alpha}}{1-CK^2\abs{\alpha}}.
	\end{align*}
	One may prove the bound for $Z_{\Psi,X}$ by doing the same computation but using \eqref{norm moment bound 2} instead of \eqref{norm moment bound 1}.
	
	We next handle analyticity of $Z_{\Phi,X}$. We shall prove by induction on $r$ that we may differentiate under the integral sign to get the formula
	\begin{align} \label{MGF derivative}
	Z_{\Phi,X}^{(r)}(\alpha) = (-1)^r\E{\norm{X}_2^{2r}e^{-\alpha\norm{X}_2^2}}.
	\end{align}
	Assume the formula is true for all $r' < r$. Then
	\begin{align} \label{MGF derivative limit}
	Z_{\Phi,X}^{(r)}(\alpha) & = (-1)^{r-1}\lim_{h \to 0}\E{\frac{\norm{X}_2^{2r-2}e^{-(\alpha+h)\norm{X}_2^2}-\norm{X}_2^{2r-2}e^{-\alpha\norm{X}_2^2}}{h}} \\
	& = (-1)^r\lim_{h \to 0}\E{\norm{X}_2^{2r-2}e^{-\alpha\norm{X}_2^2}\frac{1-e^{-h\norm{X}_2^2}}{h}}
	\end{align}
	Next, note that the integrand is positive and by the mean value theorem, for a fixed value of $\norm{X}_2^2$, we have
	\[
	\frac{1-e^{-h\norm{X}_2^2}}{h} = \norm{X}_2^2e^{-h'\norm{X}_2^2}
	\]
	for some $h' \in [0,h]$ if $h > 0$ and $h' \in [h,0]$ otherwise. As such, we have
	\[
	\norm{X}_2^{2r-2}e^{-\alpha\norm{X}_2^2}\frac{1-e^{-h\norm{X}_2^2}}{h} \leq \norm{X}_2^{2r}e^{(\abs{h}-\alpha)\norm{X}_2^2}
	\]
	For $\abs{h}-\alpha \leq 1/CK^2$, one can easily show that this is integrable by expanding this as a power series in $\norm{X}_2^2$ and bounding the growth of the coefficients as above. As such, we may apply the Dominated Convergence Theorem to push the limit inside the expectation in \eqref{MGF derivative limit}, thereby yielding \eqref{MGF derivative}.
	
	In particular, differentiating $Z_{\Phi,X}$ at 0, we see that its Taylor series at 0 is given by
	\begin{align}
	Z_{\Phi,X}(\alpha) \sim \sum_{r=0}^\infty\frac{\E{\norm{X}^{2r}_2}}{r!}(-\alpha)^r.
	\end{align}
	The formula above shows that the Taylor series is absolutely convergent on our chosen interval. We next need to show that $Z_{\Phi,X}$ agrees with its Taylor series on this interval, meaning we have to show that the remainder term for the $r$-th Taylor polynomial goes to zero pointwise. The Lagrange form of the remainder term is written as
	\[
	R_{Z_{\Phi,X},r}(\alpha) = \frac{Z_{\Phi,X}^{(r+1)}(\alpha')}{(r+1)!}\alpha^{r+1}
	\]
	where $0 < \abs{\alpha'} < \abs{\alpha}$. Applying Cauchy-Schwarz to the formula \eqref{MGF derivative}, we get
	\begin{align}
	\abs{Z_{\Phi,X}^{(r+1)}(\alpha)} \leq \paren*{\E{\norm{X}_2^{4r+2}}}^{1/2}\paren*{\E{e^{-2\alpha\norm{X}_2^2}}}^{1/2}.
	\end{align}
	Lemma \ref{norm concentration} again allows us to compute
	\[
	\paren*{\E{\norm{X}_2^{4r+2}}}^{1/2} \leq (CK^2)^{r+1}\paren{n^{r+1}+(r+1)!}.
	\]
	This implies that for any $C' > 2C$,
	\begin{align*}
	\norm{R_{Z_{\Phi,X},r}}_{L_\infty([-1/C'K^2,1/C'K^2])} & \leq \frac{\norm{Z_{\Phi,X}^{(r+1)}}_{L_\infty([-1/C'K^2,1/C'K^2])}}{(C'K^2)^{r+1}(r+1)!} \\
	& \leq \frac{(CK^2)^{r+1}\paren{n^{r+1}+(r+1)!}}{(C'K^2)^{r+1}(r+1)!}\paren*{\E{e^{2\norm{X}_2^2/C'K^2}}}^{1/2}  \\
	& \leq \paren*{\frac{C}{C'}}^{r+1}\paren*{\frac{n^{r+1}}{(r+1)!}+1} \paren*{e^{2nC/C'} + \frac{2C/C'}{1-2C/C'}}.
	\end{align*}
	Using the fact that $r! \sim \paren{\frac{r}{e}}^r$, this last expression decays to zero as $r$ tends to $\infty$. Finally, to prove the claim for $Z_{\Psi,X}$, we repeat the same arguments.
\end{proof}

Note that in the course of proving the last lemma, we have also proved the following result to be used elsewhere in the paper.

\begin{lemma}[Taylor remainder terms for $Z_{\Phi,X}$ and $Z_{\Psi,X}$] \label{MGF remainder}
	Let $X$ be a subgaussian random vector in $\R^n$ with subgaussian norm bounded above by $K \geq 1$. There is an absolute constant $C$ such that for all $0 < \alpha < 1/CK^2$, on the interval $\sqbracket*{-\alpha,\alpha}$, the remainder terms for the $r$-th degree Taylor polynomials for $Z_{\Phi,X}$ and $Z_{\Psi,X}$ at 0 satisfy the uniform bound
	\begin{align}
	\norm{R_{Z_{\Phi,X},r}}_{\infty}, \norm{R_{Z_{\Psi,X},r}}_{\infty}\leq (CK^2)^{r+1}\alpha^{r+1}\paren*{\frac{n^{r+1}}{(r+1)!}+1} \paren*{e^{CK^2\alpha n} + \frac{CK^2\alpha}{1-CK^2\alpha}}
	\end{align}
\end{lemma}

\begin{proof}[Proof of Lemma \ref{lem: partition functions characterize gaussians}]
	By Lemma \ref{lem: analyticity}, all four functions are analytic in a neighborhood of 0. Now recall that two different analytic functions cannot agree on a sequence with an accumulation point.
\end{proof}

We now move on to proving Theorem \ref{thm: second gaussian test, robust}. This requires the following technical lemma.

\begin{lemma} \label{lem: second gaussian test, robust lemma}
	Let $X$ be subgaussian random vector in $\R^n$ with subgaussian norm bounded above by $K \geq 1$. Suppose the moments of $\norm{X}_2^2$ and $\norm{g}_2^2$ agree up to order $r-1$, but there is a number $\Delta > 0$ such that $\abs*{\E{ \norm{X}_2^{2r}} - \E{ \norm{g}_2^{2r}}} \geq \Delta$, then there is an absolute constant $C$ such that for $\abs{\alpha} \leq \Delta r/(CK^2)^r(n^{r+1}+(r+1)!)$, we have
	\begin{align} \label{CGF difference 1}
	\abs*{(\log Z_{\Phi,X})'\paren*{\alpha}-(\log Z_{\Phi,g})'\paren*{\alpha}} \geq \frac{\Delta}{2(r-1)!}\abs*{\alpha}^{r-1}.
	\end{align}
	Similarly, suppose the moments of $\inprod{X,X'}$ and $\inprod{g,g'}$ agree up to order $r-1$ but $\abs*{\E{ \inprod{X,X'}^r} - \E{ \inprod{g,g'}^r}} \geq \Delta$, then for $\abs{\alpha} \leq \Delta r/(CK^2)^r(n^{r+1}+(r+1)!)$, we have
	\begin{align} \label{CGF difference 2}
	\abs*{(\log Z_{\Psi,X})'\paren*{\alpha}-(\log Z_{\Psi,g})'\paren*{\alpha}} \geq \frac{\Delta}{2(r-1)!}\abs*{\alpha}^{r-1}.
	\end{align}
\end{lemma}

\begin{proof}
	Let us first prove \eqref{CGF difference 1}. For every positive integer $k$, let $p_{X,k}(\alpha) = \sum_{j=0}^k \E{\norm{X}_2^{2j}}\alpha^j/j!$ denote the $k$-th Taylor polynomial of $Z_{\Phi,X}$, and define $p_{g,k}$ analogously. For convenience, also denote the $k$-th Taylor remainder term as $R_{X,k} := R_{Z_{\Phi,X},k}$. For any $\alpha$, we then have
	\begin{equation}
	(\log Z_{\Phi,X})'(\alpha)-(\log Z_{\Phi,g})'(\alpha) = \frac{Z_{\Phi,X}'(\alpha)}{Z_{\Phi,X}(\alpha)} - \frac{Z_{\Phi,g}'(\alpha)}{Z_{\Phi,g}(\alpha)},
	\end{equation}
	which we can then bound using
	\begin{align} \label{CGF difference}
	\abs*{\frac{Z_{\Phi,X}'(\alpha)}{Z_{\Phi,X}(\alpha)} - \frac{Z_{\Phi,g}'(\alpha)}{Z_{\Phi,g}(\alpha)}} \geq \abs*{\frac{p_{X,r}'(\alpha)}{p_{X,r-1}(\alpha)} - \frac{p_{g,r}'(\alpha)}{p_{g,r-1}(\alpha)}} - \abs*{\frac{Z_{\Phi,X}'(\alpha)}{Z_{\Phi,X}(\alpha)} - \frac{p_{X,r}'(\alpha)}{p_{X,r-1}(\alpha)}} - \abs*{\frac{p_{g,r}'(\alpha)}{p_{g,r-1}(\alpha)} - \frac{Z_{\Phi,g}'(\alpha)}{Z_{\Phi,g}(\alpha)}}.
	\end{align}
	
	We now bound each of these three terms individually. First, we need upper and lower bounds for $p_{X,k}(\alpha)$. Using the $\norm{X}_2^2$ moment bound \eqref{alternate norm moment bound}, we have
	\begin{align*}
	\abs*{p_{X,k}(\alpha) -1} & \leq \sum_{j=1}^k \frac{\E{\norm{X}_2^{2j}}\abs*{\alpha}^j}{j!} \\
	& \leq \sum_{j=1}^k \frac{(CK^2)^{j}(n^j+j!)\abs*{\alpha}^j}{j!} \\
	& = \sum_{j=1}^k \frac{(CK^2n\abs{\alpha})^j}{j!} + \sum_{j=1}^k (CK^2\abs{\alpha})^j \\
	& \leq e^{CK^2n\abs{\alpha}} + \frac{CK^2\abs{\alpha}}{1-CK^2\abs{\alpha}}.
	\end{align*}
	By sharpening the constant $C$ in our assumption on $\abs{\alpha}$ if necessary, we may thus ensure that
	\begin{align} \label{Taylor poly bounds}
	\abs*{p_{X,k}(\alpha) -1} \leq \frac{1}{2}
	\end{align}
	By the same argument, we can also ensure that
	\begin{align} \label{Taylor poly derivative bounds}
	\abs*{p_{X,k}'(\alpha) - \E\norm{X}^2_2} \leq \frac{1}{2}.
	\end{align}
	
	By our assumptions on the moments of $\norm{X}_2^2$ and $\norm{g}_2^2$, we have $p_{X,r-1} \equiv p_{g,r-1}$. Furthermore, only the leading terms of $p_{X,r}'$ and $p_{g,r}'$ differ. This, together with \eqref{Taylor poly bounds} implies that
	\begin{align} \label{Taylor difference}
	\abs*{\frac{p_{X,r}'(\alpha)}{p_{X,r-1}(\alpha)} - \frac{p_{g,r}'(\alpha)}{p_{g,r-1}(\alpha)}} & \geq \frac{2}{3} \abs*{p_{X,r}'(\alpha) - p_{g,r}'(\alpha) } \nonumber \\
	& \geq \frac{2\Delta\abs{\alpha}^{r-1}}{3(r-1)!}.
	\end{align}
	Next, we have
	\begin{align} \label{CGF remainder bound}
	\abs*{\frac{Z_{\Phi,X}'(\alpha)}{Z_{\Phi,X}(\alpha)} - \frac{p_{X,r}'(\alpha)}{p_{X,r-1}(\alpha)}} \leq \abs*{\frac{Z_{\Phi,X}'(\alpha)}{Z_{\Phi,X}(\alpha)} - \frac{p_{X,r}'(\alpha)}{p_{X,r}(\alpha)}} + \abs*{\frac{p_{X,r}'(\alpha)}{p_{X,r}(\alpha)} - \frac{p_{X,r}'(\alpha)}{p_{X,r-1}(\alpha)}}.
	\end{align}
	
	Again we bound these two terms individually. Using the identity $p_{X,r}(\alpha) = p_{X,r-1}(\alpha) + \E\norm{X}_2^{2r}(-\alpha)^r/r!$, we get
	\begin{align}
	\abs*{\frac{p_{X,r}'(\alpha)}{p_{X,r}(\alpha)} - \frac{p_{X,r}'(\alpha)}{p_{X,r-1}(\alpha)}} & = \abs*{\frac{p_{X,r}'(\alpha)}{p_{X,r}(\alpha)}}\abs*{1- \frac{p_{X,r}(\alpha)}{p_{X,r-1}(\alpha)}} \nonumber \\
	& = \abs*{\frac{p'_{X,r}(\alpha)}{p_{X,r}(\alpha)p_{X,r-1}(\alpha)}}\frac{\E\norm{X}_2^{2r}\abs{\alpha}^r}{r!}
	\end{align}
	Using the bounds on $p_{X,r}$ and $p_{X,r}'$ \eqref{Taylor poly bounds} and \eqref{Taylor poly derivative bounds}, together with the $\norm{X}_2^2$ moment bound \eqref{alternate norm moment bound}, we get
	\begin{align} \label{Taylor one term off bound}
	\abs*{\frac{p'_{X,r}(\alpha)}{p_{X,r}(\alpha)p_{X,r-1}(\alpha)}}\frac{\E\norm{X}_2^{2r}\abs{\alpha}^r}{r!} & \leq \frac{3}{8}\frac{(CK^2)^r(n^r+r!)\abs{\alpha}^r}{r!}.
	\end{align}
	
	For the first term in \eqref{CGF remainder bound}, we write
	\begin{align} \label{CGF Taylor deviation}
	\abs*{\frac{Z_{\Phi,X}'(\alpha)}{Z_{\Phi,X}(\alpha)} - \frac{p_{X,r}'(\alpha)}{p_{X,r}(\alpha)}} & = \abs*{\paren*{\log Z_{\Phi,X}(\alpha)}'- \paren*{\log p_{X,r}(\alpha)}'} \nonumber \\
	& = \abs*{\frac{d}{d\alpha} \log\paren*{\frac{Z_{\Phi,X}(\alpha)}{p_{X,r}(\alpha)}} } \nonumber \\
	& = \abs*{\frac{d}{d\alpha} \log\paren*{1+\frac{R_{X,r}(\alpha)}{p_{X,r}(\alpha)}} }.
	\end{align}
	Using Lemma \ref{MGF remainder} together with our assumptions on $\abs{\alpha}$, we observe that
	\begin{align} \label{Taylor remainder bound}
	\abs{R_{X,r}(\alpha)} & \leq (CK^2)^{r+1}\abs{\alpha}^{r+1}\paren*{\frac{n^{r+1}}{(r+1)!}+1} \paren*{e^{CK^2\abs{\alpha} n} + \frac{CK^2\abs{\alpha}}{1-CK^2\abs{\alpha}}} \nonumber \\
	& \leq (CK^2)^{r+1}\abs{\alpha}^{r+1}\paren*{\frac{n^{r+1}}{(r+1)!}+1}.
	\end{align}
	In particular, by sharpening the constant $C$ in our assumption on $\abs{\alpha}$ if necessary, we can ensure that this quantity is less than $\frac{1}{4}$. In this case, we have
	\[
	\abs*{\frac{R_{X,r}(\alpha)}{p_{X,r}(\alpha)}} \leq \frac{1}{2},
	\]
	so that
	\begin{align} \label{log derivative bound}
	\abs*{\frac{d}{d\alpha} \log\paren*{1+\frac{R_{X,r}(\alpha)}{p_{X,r}(\alpha)}} } & = \abs*{\log'\paren*{1+\frac{R_{X,r}(\alpha)}{p_{X,r}(\alpha)}} }\abs*{\paren*{\frac{R_{X,r}(\alpha)}{p_{X,r}(\alpha)}}'} \nonumber \\
	& \leq 2\abs*{\paren*{\frac{R_{X,r}(\alpha)}{p_{X,r}(\alpha)}}'} \nonumber \\
	& \leq 2 \paren*{\abs*{\frac{R_{X,r}'(\alpha)}{p_{X,r}(\alpha)}} + \abs*{\frac{R_{X,r}(\alpha)p_{X,r}'(\alpha)}{p_{X,r}(\alpha)^2}}}.
	\end{align}
	
	By our bounds on these functions \eqref{Taylor poly bounds}, \eqref{Taylor poly derivative bounds}, and \eqref{Taylor remainder bound}, we have
	\begin{align}
	\abs*{\frac{R_{X,r}(\alpha)p_{X,r}'(\alpha)}{p_{X,r}(\alpha)^2}} \leq (CK^2)^{r+1}\abs{\alpha}^{r+1}\paren*{\frac{n^{r+1}}{(r+1)!}+1}.
	\end{align}
	Furthermore, by using the moment bounds \eqref{alternate norm moment bound} as before, one can show that
	\[
	\abs{R_{X,r}'(\alpha)} \leq (CK^2)^r\abs{\alpha}^r\paren*{\frac{n^{r+1}}{r!}+r+1}.
	\]
	so that the first term is also bounded according to
	\begin{align}
	\abs*{\frac{R_{X,r}'(\alpha)}{p_{X,r}(\alpha)}} \leq (CK^2)^r\abs{\alpha}^r\paren*{\frac{n^{r+1}}{r!}+r+1}.
	\end{align}
	As such, combining \eqref{CGF Taylor deviation} and \eqref{log derivative bound} tells us that 
	\begin{align}
	\abs*{\frac{Z_{\Phi,X}'(\alpha)}{Z_{\Phi,X}(\alpha)} - \frac{p_{X,r}'(\alpha)}{p_{X,r}(\alpha)}} & \leq (CK^2)^r\abs{\alpha}^r\paren*{\frac{n^{r+1}}{r!}+r+1} + (CK^2)^{r+1}\abs{\alpha}^{r+1}\paren*{\frac{n^{r+1}}{(r+1)!}+1} \nonumber \\
	& \leq (CK^2)^r\abs{\alpha}^r\paren*{\frac{n^{r+1}}{r!}+r+1}.
	\end{align}
	We can now use this estimate together with \eqref{Taylor one term off bound} to continue \eqref{CGF remainder bound}, writing
	\begin{align}
	\abs*{\frac{Z_{\Phi,X}'(\alpha)}{Z_{\Phi,X}(\alpha)} - \frac{p_{X,r}'(\alpha)}{p_{X,r-1}(\alpha)}} & \leq (CK^2)^r\abs{\alpha}^r\paren*{\frac{n^{r+1}}{r!}+r+1} +  \frac{(CK^2)^r(n^r+r!)\abs{\alpha}^r}{r!} \nonumber \\
	& \leq (CK^2)^r\abs{\alpha}^r\paren*{\frac{n^{r+1}}{r!}+r+1}.
	\end{align}
	
	Notice that same methods also give us
	\begin{align}
	\abs*{\frac{p_{g,r}'(\alpha)}{p_{g,r-1}(\alpha)} - \frac{Z_{\Phi,g}'(\alpha)}{Z_{\Phi,g}(\alpha)}} \leq (CK^2)^r\abs{\alpha}^r\paren*{\frac{n^{r+1}}{r!}+r+1}.
	\end{align}
	We may therefore finally substitute these last two bounds, together with \eqref{Taylor difference}, into \eqref{CGF difference}. This yields
	\begin{align}
	\abs*{(\log Z_{\Phi,X})'(\alpha)-(\log Z_{\Phi,g})'(\alpha)} \geq \frac{2\Delta\abs{\alpha}^{r-1}}{3(r-1)!} - C(CK^2)^r\abs{\alpha}^r\paren*{\frac{n^{r+1}}{r!}+r+1}.
	\end{align}
	We now claim that with our assumptions on $\abs{\alpha}$, the first term dominates the second. This is a simple calculation, thereby competing the proof of \eqref{CGF difference 1}. To prove \eqref{CGF difference 2}, we repeat the entire argument, but using the relevant estimates for $Z_{\Psi,X}$ instead of those for $Z_{\Phi,X}$.
\end{proof}

Applying the previous lemma in the setting of our NGCA model, we get the following result.

\begin{theorem}[Robustness for non-gaussian eigenvalues] \label{Phi and Psi difference theorem}
	Let $X$ be a subgaussian random vector satisfying the NGCA model \eqref{NGCA model}, and with subgaussian norm bounded above by $K \geq 1$. Let $\lambda_1(\Phi_{\tilde{X},\alpha}),\ldots,\lambda_d(\Phi_{\tilde{X},\alpha})$ denote the eigenvalues of $\Phi_{\tilde{X},\alpha}$. Suppose the moments of $\norm{\tilde{X}}_2^2$ and $\norm{g_d}_2^2$ agree up to order $r-1$, but there is a number $\Delta > 0$ such that $\abs*{\E{ \norm{\tilde{X}}_2^{2r}} - \E{ \norm{g_d}_2^{2r}}} \geq \Delta$, then there is an absolute constant $C$ such that for $\abs{\alpha} \leq \Delta r/(CK^2)^r(d^{r+1}+(r+1)!)$, we have
	\begin{align} \label{Phi difference}
	\abs*{\frac{1}{d}\sum_{i=1}^d\lambda_i(\Phi_{\tilde{X},\alpha})-\frac{1}{2\alpha+1}} \geq \frac{\Delta}{2d(r-1)!}\abs{\alpha}^{r-1}.
	\end{align}
	Similarly, let $\lambda_1(\Psi_{\tilde{X},\alpha}),\ldots,\lambda_d(\Psi_{\tilde{X},\alpha})$ denote the eigenvalues of $\Psi_{\tilde{X},\alpha}$, and suppose the moments of $\inprod{X,X'}$ and $\inprod{g,g'}$ agree up to order $r-1$ but $\abs*{\E{ \inprod{X,X'}^r} - \E{ \inprod{g,g'}^r}} \geq \Delta$. Then for $\abs{\alpha} \leq \Delta r/(CK^2)^r(d^{r+1}+(r+1)!)$, we have
	\begin{align} \label{Psi difference}
	\abs*{\frac{1}{d}\sum_{i=1}^d\lambda_i(\Psi_{\tilde{X},\alpha})-\frac{\alpha}{\alpha^2-1}} \geq \frac{\Delta}{2d(r-1)!}\abs{\alpha}^{r-1}.
	\end{align}
\end{theorem}

\begin{proof}
	This is simply a translation of the previous theorem with the help of Lemma \ref{lem: trace of Phi and Psi}, which tells us that the log derivatives of the partition functions are equal to the traces of $\Phi_{X,\alpha}$ and $\Psi_{X,\alpha}$, and that of Lemma \ref{lem: formula for Phi and Psi for gaussian}, which tells us what the gaussian eigenvalue is.
\end{proof}

\begin{proof}[Proof of Theorem \ref{thm: second gaussian test, robust}]
	Combine the previous Corollary with Theorem \ref{thm: first gaussian test, robust}.
\end{proof}

\section{Identities for $\Phi_g$ and $\Psi_g$}

In this section, we let $g_1$ denote a standard gaussian random variable, and $g_n$, a standard gaussian random vector in $\R^n$. First, notice that independence gives $Z_{\Phi,g_n}(\alpha) = Z_{\Phi,g_1}(\alpha)^n$ and $Z_{\Psi,g_n}(\alpha) = Z_{\Psi,g_1}(\alpha)^n$.

\begin{lemma}
	We have the identities $Z_{\Phi,g_n}(\alpha) = \paren*{2\alpha+1}^{-n/2}$ when $\alpha > -1/2$ and $Z_{\Psi,g_n}(\alpha) = \paren*{1-\alpha^2}^{-n/2}$ when $\abs{\alpha} < 1$.
\end{lemma}

\begin{proof}
	By the remarks above, it suffices to prove the formula when $n=1$. These are then simple exercises in calculus. Notice that
	\[
	Z_{\Phi,g_1}(\alpha) = \E{ e^{-\alpha g_1^2}} = \frac{1}{\sqrt{2\pi}}\int_{-\infty}^\infty e^{-\alpha t^2} e^{-\frac{t^2}{2}} dt = \frac{1}{\sqrt{2\pi}}\int_{-\infty}^\infty e^{-\frac{(2\alpha+1)t^2}{2}} dt.
	\]
	Now substitute $u = \sqrt{2\alpha+1}\cdot t$ to arrive at the formula for $Z_{\Phi,g_1}$. For the next formula, we use conditional expectations to write
	\begin{align} \label{psi conditional expectation}
	Z_{\Psi,g_1}(\alpha) = \E{ e^{-\alpha g_1 g_1'}} = \E\sqbracket{ \E\sqbracket{e^{-\alpha g_1 g_1'} \big| g_1}}.
	\end{align}
	The inner expectation can be computed as
	\[
	\E\sqbracket{e^{-\alpha g_1 g_1'} \big| g_1} = \frac{1}{\sqrt{2\pi}}\int_{-\infty}^\infty e^{-\alpha g_1 t} e^{-\frac{t^2}{2}} dt = e^{\frac{(\alpha g_1)^2}{2}}.
	\]
	Substituting this back into \eqref{psi conditional expectation} and using the same technique as above gives us what we want.
\end{proof}

\begin{lemma} \label{lem: formula for log derivatives for gaussian}
	We have the identities $-(\log Z_{\Phi,g_n})'(\alpha) = n(2\alpha+1)^{-1}$ when $\alpha > -1/2$ and \\ $-(\log Z_{\Psi,g_n})'(\alpha) = n\alpha(\alpha^2-1)^{-1}$ when $\abs{\alpha} < 1$.
\end{lemma}

\begin{lemma} \label{lem: formula for Phi and Psi for gaussian}
	We have the identities $\Phi_{g_n,\alpha} = (2\alpha+1)^{-1}I_n$ when $\alpha > -1/2$ and $\Psi_{g_n,\alpha} = \alpha(\alpha^2-1)^{-1}I_n$ when $\abs{\alpha} < 1$. Here, $I_n$ is the $n$-dimensional identity matrix.
\end{lemma}
\begin{proof}
	By rotational symmetry, we know that both matrices are multiples of the identity. To compute these scalars, it hence suffices to find the trace of both matrices. But
	\[
	\tr{\Phi_{g_n,\alpha}} = \frac{\E{e^{-\alpha\norm{g_n}_2^2}\norm{g_n}_2^2}}{\E e^{-\alpha\norm{g_n}_2^2}} = -(\log Z_{\Phi,g_n})'(\alpha).
	\]
	Dividing by $n$ and using the previous lemma gives us what we want.
\end{proof}

\section{Concentration and moment bounds} \label{concentration of estimators}

\begin{theorem}[Concentration of norm for general subgaussian vectors] \label{norm concentration}
	Let $X$ be a subgaussian random vector in $\R^n$, with $\norm{X}_{\psi_2} \leq K$. There is a universal constant $C$ such that for each positive integer $r > 0$, the moments of $\norm{X}_2$ and $\inprod{X,X'}$ satisfy
	\begin{align} \label{norm moment bound 1}
	\paren*{\E{\norm{X}_2^r}}^{1/r} \leq CK\paren{\sqrt{n} + \sqrt{r}}
	\end{align}
	\begin{align} \label{norm moment bound 2}
	\paren*{\E{\abs{\inprod{X,X'}}^r}}^{1/{2r}} \leq  CK\paren{\sqrt{n} + \sqrt{r}}.
	\end{align}
\end{theorem}

\begin{proof}
	The second bound follows from the first, since by Cauchy-Schwarz,
	\[
	\paren*{\E{\abs{\inprod{X,X'}}^r}}^{1/{2r}} \leq \paren*{\E{\norm{X}_2^r\norm{X'}_2^r}}^{1/2r} = \paren*{\E{\norm{X}_2^r}}^{1/r}
	\]
	To prove \eqref{norm moment bound 1}, pick a $\frac{1}{2}$-net $\mathcal{N}$ on $S^{n-1}$. A volumetric argument shows that one may pick $\mathcal{N}$ to have size no more than $5^n$ \cite{V}. We then have
	\[
	\norm{X} = \sup_{v \in S^{n-1}}\inprod{X,v} \leq 2\sup_{v \in \mathcal{N}}\inprod{X,v}.
	\]
	By definition, there is a universal constant $c$ such that for any fixed unit vector $v \in S^{d-1}$, $\Prob[]{ \inprod{X,v} > t } \leq 2\exp\paren*{-\frac{ct^2}{K^2}}$. Taking a union bound over the net thus gives
	\begin{align} \label{norm tail}
	\Prob[]{\norm{X} > 2t} \leq 2\exp\paren*{n\log5 - \frac{ct^2}{K^2}}.
	\end{align}
	Next, we integrate out the tail bound \eqref{norm tail} to obtain bounds for the moments. Observe that if $\frac{ct^2}{2K^2} \geq n\log 5$, we have $n\log 5 - \frac{ct^2}{K^2} \leq - \frac{ct^2}{2K^2}$. This condition on t is equivalent to $t \geq CK\sqrt{n}$, so we have
	\begin{align} \label{norm tail bound}
	\Prob[]{\norm{X} > 2t} \leq
	\begin{cases}
	1 & t < CK\sqrt{n} \\
	2\exp\paren*{-\frac{ct^2}{K^2}} & t \geq CK\sqrt{n}
	\end{cases}
	\end{align}
	For any positive integer $r$, we integrate this bound to get
	\begin{align*}
	\E{\norm{X}_2^r} & = \int_0^\infty rt^{r-1}\Prob[]{\norm{X} > t}dt \\
	& \leq \int_0^{CK\sqrt{n}} rt^{r-1} dt + \int_{CK\sqrt{n}}^\infty 2rt^{r-1}\exp\paren*{-\frac{ct^2}{K^2}}dt \\
	& \leq C^rK^rn^{r/2} + C^rK^rr\int_0^\infty t^{r/2-1}e^{-t}dt.
	\end{align*}
	The integral in the last line is the gamma function, so in short, we have shown that
	\begin{align} \label{alternate norm moment bound}
	\E{\norm{X}_2^r} \leq C^rK^r\paren{n^{r/2}+\Gamma(r/2+1)}.
	\end{align}
	Taking $r$-th roots of both sides and using H{\"o}lder, together with the fact that $\Gamma(x)^{1/x} \lesssim x$, gives \eqref{norm moment bound 1}.
\end{proof}

\begin{lemma}[Covariance estimation for subgaussian random vectors] \label{covariance estimation}
	Let $X$ be a centered subgaussian random vector in $\R^n$ with covariance matrix $\Sigma$ and subgaussian norm satisfying $\norm{X}_{\psi_2} \leq K$ for some $K \geq 1$. Let $\hat{\Sigma}_N=\frac{1}{N}\sum_{i=1}^N X_iX_i^T$ denote the sample covariance matrix from $N$ independent samples. Then there is an absolute constant $C$ such that for any $0 < \epsilon, \delta < 1$, we have $\Prob[]{\norm*{\hat{\Sigma}_N-\Sigma} > \epsilon} \leq \delta$ so long as $N \geq CK^2(n+\log(1/\delta))\epsilon^{-2}$.
\end{lemma}

\begin{proof}
	Refer to \cite{Vershynin2011b}.
\end{proof}

\begin{lemma}[Moments of spherical marginals]
	Let $\theta$ be uniformly distributed on the sphere $S^{n-1}$. Then for any unit vector $v \in S^{n-1}$ and any positive integer $k$, we have
	\begin{align} \label{spherical moments}
	\E{\inprod{\theta,v}^{2k}} = \frac{1\cdot 3\cdots(2k-1)}{n\cdot(n+2)\cdots(n+2k-2)}
	\end{align}
\end{lemma}

\begin{proof}
	There are several ways to prove this identity. We shall prove this by computing gaussian integrals. Let $g_1$ and $g_n$ denote standard gaussians in 1 dimension and $n$ dimensions respectively. Then using the radial symmetry of $g$, we have
	\[
	\E{g_1^{2k}} = \E{\inprod{g_n,v}^{2k}} = \E{\inprod{\norm{g_n}_2\theta,v}^{2k}} = \E{\norm{g_n}_2^{2k}}\E{\inprod{\theta,v}^{2k}}.
	\]
	Rearranging gives
	\[
	\E{\inprod{\theta,v}^{2k}} = \frac{\E{g_1^{2k}}}{\E{\norm{g_n}_2^{2k}}}.
	\]
	We then compute
	\begin{align} \label{gaussian norm moment}
	\E{\norm{g_n}_2^{2k}} = \frac{\omega_n}{(2\pi n)^{n/2}}\int_0^\infty r^{2k}r^{n-1}e^{-r^2/2}dr,
	\end{align}
	where $\omega_n$ is the volume of the sphere $S^{n-1}$. It is well known that
	\[
	\omega_n = \frac{2\pi^{n/2}}{\Gamma(n/2)},
	\]
	while we also have
	\[
	\int_0^\infty r^{2k}r^{n-1}e^{-r^2/2}dr = 2^{n/2+k-1}\Gamma(n/2+k).
	\]
	Substituting these back into \eqref{gaussian norm moment} gives
	\begin{equation} \label{eq: gaussian moments}
	\E{\norm{g_n}_2^{2k}} = 2^k\frac{\Gamma(n/2+k)}{\Gamma(n/2)} = n\cdot(n+2)\cdots(n+2k-2).
	\end{equation}
	This yields the denominator in \eqref{spherical moments}. A similar calculation for $\E{g_1^{2k}}$ yields the numerator. 
\end{proof}

\begin{proof}[Proof of Theorem \ref{Phi concentration}]
	Let $Y = e^{-\alpha\norm{X}_2^2}X$. Then $Y$ is a subgaussian random vector with $\norm{Y}_{\psi_2} \leq K$. Let $\Sigma$ and $\hat{\Sigma}$ denote its covariance and empirical covariance matrices respectively. Then $\norm{\Sigma} \leq 1$ and by Lemma \ref{covariance estimation}, we have $\norm{\hat{\Sigma} - \Sigma} \leq \epsilon/2$ with probability at least $1 - \delta/2$. Next,
	observe that $\Phi_{X,\alpha} = Z_{\Phi,X}(\alpha)^{-1}\Sigma$ and $\hat{\Phi}_{X,\alpha} = \hat{Z}_{\Phi,X}(\alpha)^{-1}\hat{\Sigma}$, where $\hat{Z}_{\Phi,X}(\alpha) = \sum_{j=1}^N e^{-\alpha\norm{X_j}_2^2}/N$. As such, we have
	\begin{align}
	\norm{\hat{\Phi}_{X,\alpha} - \Phi_{X,\alpha}} \leq \abs{\hat{Z}_{\Phi,X}(\alpha)^{-1}}\norm{\hat{\Sigma} - \Sigma} + \abs{\hat{Z}_{\Phi,X}(\alpha)^{-1}-Z_{\Phi,X}(\alpha)^{-1}}\norm{\Sigma}.
	\end{align}
	
	Combining our lower bound on $\alpha$ with the power series formula for $Z_\Phi$ from Lemma \ref{lem: analyticity}, we have $Z_{\Phi,X}(\alpha) \geq 1/2$. Furthermore, we may apply Hoeffding's inequality to see that $\abs{\hat{Z}_{\Phi,X}(\alpha)-Z_{\Phi,X}(\alpha)} \leq \epsilon/2$ with probability at least $1- \delta/2$. We can now combine all of this together to get the probability bound.
\end{proof}

\begin{proof}[Proof of Theorem \ref{Psi concentration}]
	First, define $\Sigma = \E\sqbracket{e^{-\alpha\inprod{X,X'}}X(X')^T}$ and $\hat{\Sigma} = \sum_{i=1}^N e^{-\alpha\inprod{X_i,X_i'}}\paren{X_i\paren{X_i'}^T + X_i'X_i^T}/2N$, so that $\Psi_{X,\alpha} = Z_{\Psi,X}(\alpha)^{-1}\Sigma$ and $\hat{\Psi}_{X,\alpha} = \hat{Z}_{\Psi,X}(\alpha)^{-1}\hat{\Sigma}$. As in the previous theorem, we can write
	\begin{align} \label{Psi decomposition}
	\norm{\hat{\Psi}_{X,\alpha} - \Psi_{X,\alpha}} \leq \abs{\hat{Z}_{\Psi,X}(\alpha)^{-1}}\norm{\hat{\Sigma} - \Sigma} + \abs{\hat{Z}_{\Psi,X}(\alpha)^{-1}-Z_{\Psi,X}(\alpha)^{-1}}\norm{\Sigma}.
	\end{align}
	
	This time however, we cannot immediately invoke Lemma \ref{covariance estimation} because we can no longer view $\Sigma$ and $\hat{\Sigma}$ as the covariance and empirical covariance matrices of a random vector. Nonetheless, we can follow the same proof scheme with a few adjustments.
	
	The basic idea is to use a net argument to transform the operator deviation bound into a scalar bound for random variables. Let $\mathcal{N}$ be a $\frac{1}{4}$-net on $S^{n-1}$. By a volumetric argument, we may pick $\mathcal{N}$ to have size no more than $9^n$ \cite{Vershynin2011b}. For any $n$ by $n$ real symmetric matrix $M$, we then have
	\begin{align} \label{net operator bound}
	\norm{M} = \sup_{v \in S^{n-1}}\abs{\inprod{v,Mv}} \leq 2\sup_{v \in \mathcal{N}}\abs{\inprod{v,Mv}}.
	\end{align}
	As such, by taking a union bound, we can hope to bound $\norm{\hat{\Sigma} - \Sigma}$ by bounding $\abs{\inprod{v,(\hat{\Sigma}-\Sigma)v}}$ for a fixed unit vector $v \in S^{n-1}$. Let us do just this. We have
	\begin{align*}
	\inprod{v,\hat{\Sigma}v} & = \frac{1}{N} \sum_{i=1}^N e^{-\alpha\inprod{X_i,X_i'}}\inprod{X_i,v}\inprod{X_i',v},
	\end{align*}
	so that
	\begin{align}
	\inprod{v,(\hat{\Sigma}-\Sigma)v} & = \frac{1}{N} \sum_{i=1}^N \paren{Y_i - \E Y_i},
	\end{align}
	where
	\begin{align}
	Y_i = e^{-\alpha\inprod{X_i,X_i'}}\inprod{X_i,v}\inprod{X_i',v}.
	\end{align}
	
	Observe that the $Y_i$'s are i.i.d. random variables. At this point in the proof of covariance estimation, one observes that the resulting random variables are subexponential, so one may apply Bernstein's inequality \cite{Vershynin2011b}. Unfortunately, our $Y_i$'s are not subexponetial because of the $e^{-\alpha\inprod{X_i,X_i'}}$ factor. The way we overcome this is to condition on the size of these factors being uniformly small. Indeed, by Lemma \ref{uniform bound on exp factors} to come, we have $e^{-\alpha\inprod{X_i,X_i'}} \leq e$ for all samples $i$ with probability at least $1- \delta$. We call this event $A$.
	
	Next, define $\tilde{Y}_i := Y_i1_A$. The $Y_i$'s are i.i.d random variables with subexponential norm bounded by $eK^2$. We can then apply Bernstein and our assumption on the sample size $N$ to get
	\begin{align}
	\Prob{\abs*{\frac{1}{N}\sum_{i=1}^N \paren{\tilde{Y}_i - \E \tilde{Y}_i}} > \epsilon} \leq e^{-N\epsilon^2/CK^4} \leq \frac{\delta}{9^n}.
	\end{align}
	Conditioning on the set $A$, we have $Y_i = \tilde{Y}_i$ for each $i$. We can also rewrite the bound on the right hand side using our assumption on $N$. Doing this gives us
	\begin{align} \label{intermediate Y deviation}
	\Prob[]{\abs*{\frac{1}{N}\sum_{i=1}^N \paren{Y_i - \E \tilde{Y}_i}} > \epsilon \ \Big| \ A} \leq \frac{\delta}{9^n}.
	\end{align}
	We would like to replace $\E \tilde{Y}_i$ with $\E Y_i$, but the two quantities are not necessarily equal. Nonetheless, we can bound their difference as follows. We have
	\begin{align}
	\E Y_i - \E\tilde{Y}_i = \E Y1_{A^c} = \E\sqbracket{e^{-\alpha\inprod{X_i,X_i'}}\inprod{X_i,v}\inprod{X_i',v}1_{A^c}}.
	\end{align}
	We apply generalized H{\"o}lder to write
	\begin{align}
	\abs{\E\sqbracket{e^{-\alpha\inprod{X_i,X_i'}}\inprod{X_i,v}\inprod{X_i',v}1_{A^c}}} & \leq \paren*{\E e^{-4\alpha\inprod{X_i,X_i'}}}^{1/4} \paren*{\E\sqbracket{\inprod{X_i,v}^4\inprod{X_i',v}^4}}^{1/4}\Prob[]{A^c}^{1/2}.
	\end{align}
	We now use the moment bounds for subgaussian random variables and Lemma \ref{better bound for psi} to bound the first two multiplicands on the right. This gives us
	\begin{align}
	\abs{\E\sqbracket{e^{-\alpha\inprod{X_i,X_i'}}\inprod{X_i,v}\inprod{X_i',v}1_{A^c}}} \leq CK^2\Prob[]{A^c}^{1/2}.
	\end{align}
	
	Next, we use Lemma \ref{uniform bound on exp factors} together with our assumption on $\abs{\alpha}$, tightening the constant if necessary, to see that $\Prob[]{A^c} \leq \epsilon^2/C^2K^4$. We combine this together with the last few equations to obtain $\abs{\E Y_i - \E\tilde{Y}_i} \leq \epsilon$, and combining this with \eqref{intermediate Y deviation}, we obtain
	\begin{align}
	\Prob[]{\abs*{\frac{1}{N}\sum_{i=1}^N \paren{Y_i - \E Y_i}} > 2\epsilon \ \Big| \ A} \leq \frac{\delta}{9^n}.
	\end{align}
	Recall that $Y_i$'s were defined for a fixed $v \in \mathcal{N}$. We can take a union bound over all vectors in $\mathcal{N}$ to get
	\begin{align}
	\Prob[]{ \sup_{v \in \mathcal{N}}\abs{\inprod{v,(\hat{\Sigma}-\Sigma)v}} > 2\epsilon \ \Big| \ A} \leq \delta.
	\end{align}
	Combining this with \eqref{net operator bound} then gives
	\begin{align}
	\Prob[]{ \norm{\hat{\Sigma} - \Sigma } > 4\epsilon \ \Big| \ A} \leq \delta.
	\end{align}
	
	Let us continue to bound the other terms in \eqref{Psi decomposition} conditioned on the set $A$. Notice that on this set, $\hat{Z}_{\Psi,X}(\alpha)$ is an average of terms that are each bounded in absolute value by $e$. Using Hoeffding's inequality together with a similar argument as above to bound $\abs{\E\hat{Z}_{\Psi,X}(\alpha)1_A - Z_{\Psi,X}(\alpha)}$, one may show that
	\begin{align}
	\Prob[]{\abs{\hat{Z}_{\Psi,X}(\alpha) - Z_{\Psi,X}(\alpha)} > \epsilon/2 \  \big| \ A } \leq \delta.
	\end{align}
	We may also use the power series formula for $Z_{\Psi,X}$ from Lemma \ref{lem: analyticity} together with our bound on $\abs{\alpha}$ to show that $Z_{\Psi,X}(\alpha) \geq \frac{1}{2}$.
	
	It remains to bound $\norm{\Sigma}$. To do this, we let $v$ again be an arbitrary unit vector, and use Cauchy-Schwarz to compute
	\begin{align}
	\abs{\E\sqbracket{e^{-\alpha\inprod{X_i,X_i'}}\inprod{X_i,v}\inprod{X_i',v}}} & \leq \paren*{\E e^{-2\alpha\inprod{X_i,X_i'}}}^{1/2} \paren*{\E\sqbracket{\inprod{X_i,v}^2\inprod{X_i',v}^2}}^{1/2}.
	\end{align}
	We have already seen that moment bounds and Lemma \ref{better bound for psi} imply that this is bounded by an absolute constant $C$. In fact, we can take $C = 3$.
	
	Putting everything together, we see that on the set $A$, we can continue writing \eqref{Psi decomposition} as
	\begin{align*}
	\norm{\hat{\Psi}_{X,\alpha} - \Psi_{X,\alpha}} & \leq \abs{\hat{Z}_{\Psi,X}(\alpha)^{-1}}\norm{\hat{\Sigma} - \Sigma} + \abs{\hat{Z}_{\Psi,X}(\alpha)^{-1}-Z_{\Psi,X}(\alpha)^{-1}}\norm{\Sigma} \\
	& \leq C\epsilon.
	\end{align*}	
	Using our bound for $\Prob[]{A}$, we can therefore uncondition to get 
	\begin{align}
	\Prob[]{\norm{\hat{\Psi}_{X,\alpha} - \Psi_{X,\alpha}} > C\epsilon} \leq \delta + \Prob[]{A} \leq 2\delta.
	\end{align}
	Finally, note that we can massage the constants so that the multiplying constants in front of $\epsilon$ and $\delta$ disappear.
\end{proof}

\begin{lemma} \label{uniform bound on exp factors}
	For any $0 < \delta < 1$ and $N \in \mathbb{N}$, if $\abs{\alpha} \leq \paren*{CK^2\sqrt{\log(N/\delta)}(\sqrt{n}+\sqrt{\log(N/\delta)})}^{-1}$, then
	\begin{align}
	\Prob[]{\sup_{1 \leq i \leq N}e^{-\alpha\inprod{X_i,X_i'}} > e} \leq \delta.
	\end{align}
\end{lemma}

\begin{proof}
	Without loss of generality, assume that $\alpha > 0$. Using the union bound, it suffices to prove that
	\begin{align}
	\Prob[]{\inprod{X,X'} < -1/\alpha} = \Prob[]{e^{-\alpha\inprod{X,X'}} > e} \leq \frac{\delta}{N}.
	\end{align}
	To compute this, we first condition on $X'$ and use the subgaussian tail of $X$ to get
	\begin{align*}
	\Prob[]{\inprod{X,X'} < -1/\alpha \st X'} \leq \exp\paren*{-\frac{1}{CK^2\alpha^2\norm{X'}_2^2}},
	\end{align*}
	and integrating out $X'$, then gives
	\begin{align}
	\Prob[]{\inprod{X,X'} < - 1/\alpha} \leq \E e^{-\paren{CK^2\alpha^2\norm{X'}_2^2}^{-1}}.
	\end{align}
	
	To compute this expectation, let $A$ be the event that $\norm{X'}_2 \leq CK (\sqrt{n}+\sqrt{\log(N/\delta)})$. Then by equation \eqref{norm tail bound} in Theorem \ref{norm concentration}, we have $\Prob[]{A^c} \leq \delta/N$. As such, we can break up the expectation into the portion over $A$ and the the portion over $A^c$ to obtain
	\begin{align}
	\E e^{-\paren{CK^2\alpha^2\norm{X'}_2^2}^{-1}} & = \E\sqbracket{e^{-\paren{CK^2\alpha^2\norm{X'}_2^2}^{-1}} \st A}\Prob[]{A} + \E\sqbracket{e^{-\paren{CK^2\alpha^2\norm{X'}_2^2}^{-1}} \st A^c}\Prob[]{A^c} \nonumber \\
	& \leq  \E\sqbracket{e^{-\paren{CK^2\alpha^2\norm{X'}_2^2}^{-1}} \st A} + \Prob[]{A^c} \nonumber \\
	& \leq \exp\paren*{-\frac{1}{CK^4\alpha^2(n+\log(N/\delta))}} + \frac{\delta}{N}.
	\end{align}
	As such, we just need the first term to be less than $\delta/N$, which corresponds to the requirement that
	\begin{align*}
	\frac{1}{CK^4\alpha^2(n+\log(N/\delta))} \geq \log(N/\delta).
	\end{align*}
	This is simply a rearrangement of our assumption on $\abs{\alpha}$.
\end{proof}

\begin{lemma}[Better bound for $Z_\Psi$] \label{better bound for psi}
	There is an absolute constant $C$ such that if $\abs{\alpha} \leq 1/CK^2\sqrt{n}$, then $Z_{\Psi,X}(\alpha) \leq 3$.
\end{lemma}

\begin{proof}
	The idea of the proof is similar to that of the previous lemma. We first condition on $X'$ and use the subgaussian nature of $X$ to bound its Laplace transform, thereby obtaining
	\begin{align*}
	\E\sqbracket{e^{-\alpha\inprod{X,X'}} \st X' } \leq e^{CK^2\alpha^2\norm{X'}_2^2}.
	\end{align*}
	Integrating out $X'$ gives
	\begin{align} \label{psi intermediate bound}
	Z_{\Psi,X}(\alpha) & \leq \E e^{CK^2\alpha^2\norm{X'}_2^2} \nonumber \\
	& = \int_0^\infty \Prob[]{e^{CK^2\alpha^2\norm{X'}_2^2} > t} dt \nonumber \\
	& \leq e + \int_e^\infty \Prob[]{e^{CK^2\alpha^2\norm{X'}_2^2} > t} dt
	\end{align}
	
	Next, we use our assumption on $\abs{\alpha}$ to write
	\begin{align}
	\Prob[]{e^{CK^2\alpha^2\norm{X'}_2^2} > t} & = \Prob[]{\norm{X'}_2 > \frac{\sqrt{\log t}}{CK\abs{\alpha}} } \nonumber \\
	& \leq \Prob[]{\norm{X'}_2 > \sqrt{\log t}CK\sqrt{n} }.
	\end{align}
	For $t > e$, we have $\sqrt{\log t} > 1$, so we may apply \eqref{norm tail bound} to get
	\begin{align}
	\Prob[]{\norm{X'}_2 > \sqrt{\log t}CK\sqrt{n} } \leq e^{-\log t Cn} = t^{-Cn}.
	\end{align}
	Plugging this into \eqref{psi intermediate bound} gives
	\begin{align}
	Z_{\Psi,X}(\alpha) \leq e + \frac{e^{-Cn}}{Cn} \leq 3
	\end{align}
	if we choose $C$ to be large enough.
\end{proof}

\section{Eigenvector perturbation theory} \label{eigenvector perturbation}

If two $n$ by $n$ matrices are close in spectral norm, one can use minimax identities to show that their eigenvalues are also close. It is less trivial to show that their eigenvectors are also close, which is the case in the presence of an ``eigengap''. This is addressed by the well-known Davis-Kahan theorems \cite{Davis1970a}.

\begin{definition}
	Let $E$ and $\hat{E}$ be two subspaces of $\R^n$ of dimension $d$. Let $V$ and $\hat{V}$ be $n$ by $d$ matrices with orthonormal columns forming a basis for $E$ and $\hat{E}$ respectively. Let $\sigma_1 \geq \sigma_2 \geq \cdots \geq \sigma_d$ be the singular values of $V^T\hat{V}$. We define the \emph{principal angles} of $E$ and $\hat{E}$ to be $\theta_i(E,\hat{E}) = \arccos\sigma_i$ for $1 \leq i \leq d$.
\end{definition}

\begin{lemma}
	Let $E$, $\hat{E}$, $U$ and $\hat{U}$ be as in the previous definition. We have
	\begin{equation}
	\norm{\hat{V}\hat{V}^T - VV^T}^2_F = 2\sum_{i=1}^d \sin^2\theta_i(E,\hat{E}).
	\end{equation}
	In particular, the quantity depends only on $E$ and $\hat{E}$ and not the choice of bases.
\end{lemma}

\begin{proof}
	We expand
	\begin{align}
	\norm{\hat{V}\hat{V}^T - VV^T}^2_F = \norm{\hat{V}\hat{V}^T}_F^2 + \norm{VV^T}^2_F - 2\inprod{VV^T,\hat{V}\hat{V}^T}.
	\end{align}
	Observe that
	\begin{align}
	\norm{VV^T}^2_F = \tr{VV^TVV^T} = \tr{V^TVV^TV} = \tr{I_d} = d.
	\end{align}
	Similarly, we have
	\begin{align}
	\norm{\hat{V}\hat{V}^T}_F^2 = d.
	\end{align}
	Next, we compute
	\begin{align}
	\inprod{VV^T,\hat{V}\hat{V}^T} = \tr{VV^T\hat{V}\hat{V}^T} = \tr{\hat{V}^TVV^T\hat{V}} = \norm{\hat{V}^TV}_F^2.
	\end{align}
	Next, we use the fact that the squared Frobenius norm of a matrix is the sum of squares of its singular values to write
	\begin{align}
	\norm{\hat{V}^TV}_F^2 = \sum_{i=1}^d\sigma_i^2 = \sum_{i=1}^d \cos^2\theta_i(E,\hat{E}).
	\end{align}
	We may then combine these identities to write
	\begin{align}
	\norm{\hat{V}\hat{V}^T - VV^T}^2_F = 2\sum_{i=1}^d(1-\cos^2\theta_i(E,\hat{E})) = 2\sum_{i=1}^d\sin^2\theta_i(E,\hat{E}).
	\end{align}
	as was to be shown.
\end{proof}

Using the previous lemma, it is easy to see that the distance between subspaces is preserved under taking orthogonal complements.

\begin{lemma}
	Let $F$ and $F'$ be subspaces of $\R^n$ of dimensions $m$, and let $F'$ and $F'^\perp$ denote their orthogonal complements. We have $d(F,F') = d(F^\perp,F'^\perp)$.
\end{lemma}

We can now use these observations to state Theorem 2 from \cite{Yu2015} in a convenient form.

\begin{theorem}[Davis-Kahan theorem] \label{Davis-Kahan theorem}
	Let $\Sigma$ and $\hat{\Sigma}$ be two $n$ by $n$ symmetric real matrices, with eigenvalues $\lambda_1 \geq \cdots \geq \lambda_n$ and $\hat{\lambda}_1 \geq \cdots \geq \hat{\lambda}_n$. Fix $1 \leq r \leq s \leq n$, and assume that $\min\braces{\lambda_r-\lambda_{r+1},\lambda_s-\lambda_{s+1}} > 0$, where we define $\lambda_0 = \infty$ and $\lambda_{n+1} = - \infty$. Let $d$ = $r+n-s$, and let $V = \paren{v_1,v_2,\ldots,v_r,v_{s+1},\ldots,v_n}$ and $\hat{V} = \paren{\hat{v}_1,\hat{v}_2,\ldots,\hat{v}_r,\hat{v}_{s+1},\ldots,\hat{v}_n}$ be $n$ by $d$ matrices whose columns are orthonormal eigenvectors to $\lambda_1,\lambda_2,\ldots,\lambda_r,\lambda_{s+1},\ldots,\lambda_n$ and $\hat{\lambda}_1,\hat{\lambda}_2,\ldots,\hat{\lambda}_r,\hat{\lambda}_{s+1},\ldots,\hat{\lambda}_n$ respectively. Then
	\begin{equation}
	\norm{\hat{V}\hat{V}^T - VV^T}_F \leq \frac{2\sqrt{2d}\norm{\hat{\Sigma}-\Sigma}}{\min\braces{\lambda_r-\lambda_{r+1},\lambda_s-\lambda_{s+1}}}.
	\end{equation}
\end{theorem}

\end{document}